\documentclass{article}

\usepackage[margin=1in]{geometry}
\usepackage{amsmath, amssymb, natbib, graphicx, url}
\usepackage[utf8]{inputenc}
\usepackage{amssymb}
\usepackage{bbm}
\usepackage{hyperref}
\usepackage{cleveref}
\usepackage{xcolor}
\usepackage{multirow}
\usepackage{algorithm}% http://ctan.org/pkg/algorithms
\usepackage{algpseudocode}%

\usepackage[font=small,labelfont=bf]{caption}

\usepackage{amsthm}

\newtheorem{proposition}{Proposition}
\numberwithin{proposition}{section}
\newtheorem{lemma}[proposition]{Lemma}

\newtheorem{claim}[proposition]{Claim}

\newtheorem{remark}[proposition]{Remark}
\usepackage{thmtools}
\usepackage{thm-restate}

\newtheorem{innercustomlem}{Lemma}
\newenvironment{customlem}[1]
  {\renewcommand\theinnercustomlem{#1}\innercustomlem}
  {\endinnercustomlem}

\DeclareMathOperator{\Tr}{Tr}
\DeclareMathOperator{\Det}{Det}

\newcommand{\Lsup}{L}
\newcommand{\etabd}{\frac{1}{2r\Lsup + 2\sqrt{rm\Lmain\Lsup}}}
\newcommand{\Lmain}{L_f}
\newcommand{\Astar}{\alpha^*}
\newcommand{\Abar}{\overline{\alpha}}
\newcommand{\pmin}{p_{\text{min}}}
\newcommand{\pmax}{p_{\text{max}}}
\newcommand{\rate}{\tilde{O}\left(\left(n + \frac{\Lsup}{\mu} + \frac{\sqrt{m\Lmain\Lsup}}{\mu}\right)\log(1/\epsilon)\right)}

\newcommand{\gradi}{|\nabla_i - \nabla_i(x^*)|_2^2}

\algdef{SE}[SUBALG]{Indent}{EndIndent}{}{\algorithmicend\ }%
\algtext*{Indent}
\algtext*{EndIndent}

\usepackage{authblk}

\title{Asynchronous Distributed Optimization with Stochastic Delays}
\author{Margalit Glasgow\thanks{Department of Computer Science, Stanford University. \href{mailto:mglasgow@stanford.edu}{mglasgow@stanford.edu}} \and Mary Wootters\thanks{Departments of Computer Science and Electrical Engineering, Stanford University.  \href{mailto:marykw@stanford.edu}{marykw@stanford.edu} \\ MW and MG are supported in part by NSF award CCF-1814629, NSF CAREER award CCF-1844628, and a Sloan Research Fellowship. MG is supported by NSF award DGE-1656518}}

\begin{document}

\maketitle

\begin{abstract}%
We study asynchronous finite sum minimization in a distributed-data setting with a central parameter server. While asynchrony is well understood in parallel settings where the data is accessible by all machines---e.g., modifications of variance-reduced gradient algorithms like SAGA work well---little is known for the distributed-data setting.  We develop an algorithm ADSAGA based on SAGA for the distributed-data setting, in which the data is partitioned between many machines.
We show that with $m$ machines, under a natural stochastic delay model with an mean delay of $m$,
ADSAGA converges in $\tilde{O}\left(\left(n + \sqrt{m}\kappa\right)\log(1/\epsilon)\right)$ iterations, where $n$ is the number of component functions, and $\kappa$ is a condition number. This complexity sits squarely between the complexity $\tilde{O}\left(\left(n + \kappa\right)\log(1/\epsilon)\right)$ of SAGA \textit{without delays} and the complexity $\tilde{O}\left(\left(n + m\kappa\right)\log(1/\epsilon)\right)$ of parallel asynchronous algorithms where the delays are \textit{arbitrary} (but bounded by $O(m)$), and the data is accessible by all. Existing asynchronous algorithms with distributed-data setting and arbitrary delays have only been shown to converge in $\tilde{O}(n^2\kappa\log(1/\epsilon))$ iterations. We empirically compare on least-squares problems the iteration complexity and wallclock performance of ADSAGA to existing parallel and distributed  algorithms, including synchronous minibatch algorithms. Our results demonstrate the wallclock advantage of variance-reduced asynchronous approaches over SGD or synchronous approaches.%
\end{abstract}

\section{Introduction}
In large scale machine learning problems, distributed training has become increasingly important. In this work, we consider a distributed setting governed by a central parameter server (PS), where the training data is partitioned among a set of machines, such that each machine can only access the data it stores locally. This is common in federated learning, where the machines may be personal devices or belong to different organizations \citep{federated}. Data-partitioning can also be used in data-centers to minimize stalls from loading data from remote file systems \citep{data_stalls}. 

Asynchronous algorithms --- in which the machines do not serialize after sending updates to the PS --- are an important tool in distributed training.  Asynchrony can mitigate the challenge of having to wait for the slowest machine, which is especially important when compute resources are heterogeneous~\citep{system_hetero}. 
% While asynchronous approaches are advantageous because machines do not have to idle while for the slowest update, they result in delayed gradient updates that have been computed from stale model parameters. 
Perhaps surprisingly, there has been relatively little theoretical study of asynchronous algorithms in a distributed-data setting; most works have studied a shared-data setting where all of the data is available to all of the machines.

\begin{figure*}
\centering
\begin{tabular}{ccc}
  \includegraphics[width=4.5cm]{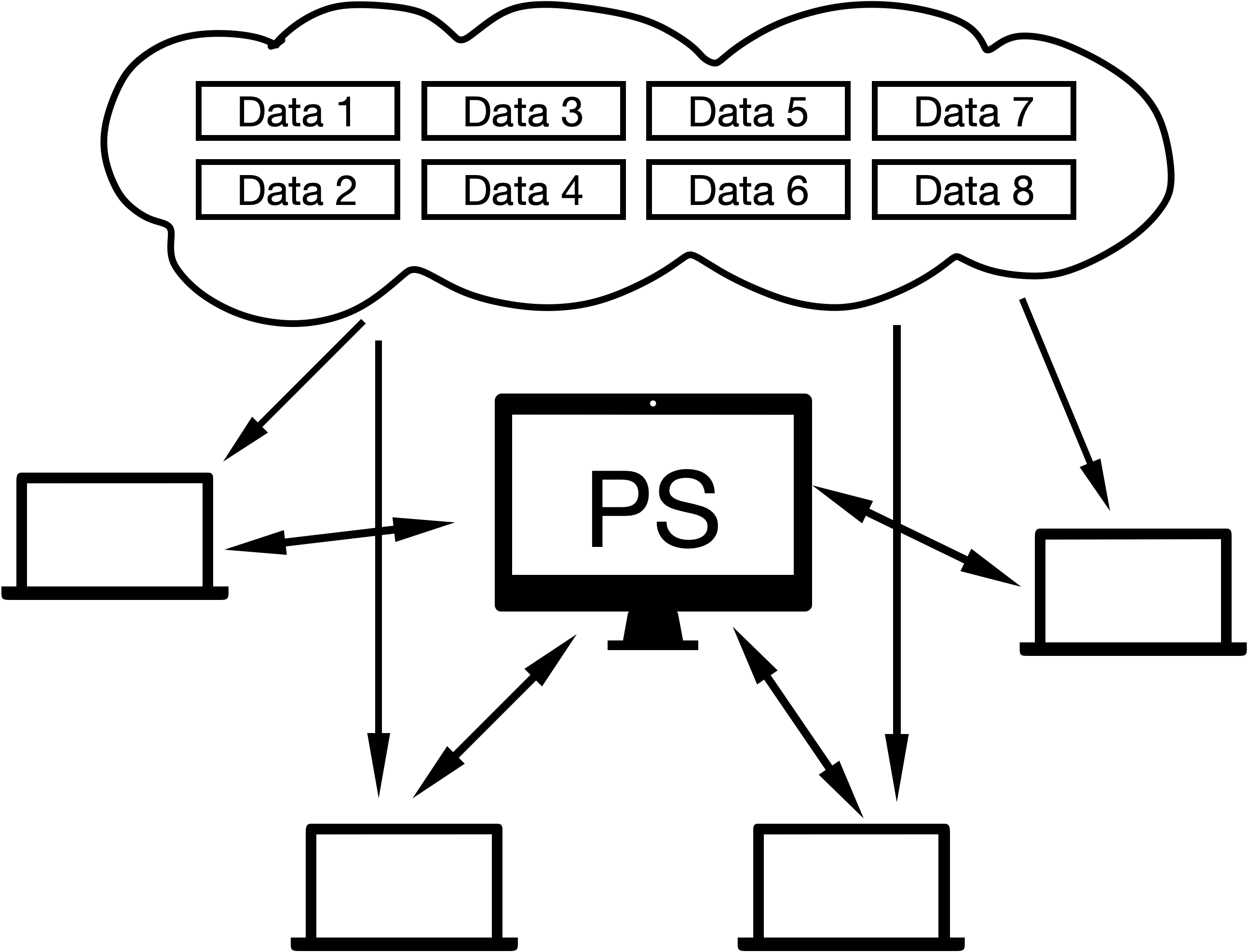} &
  \includegraphics[width=4.5cm]{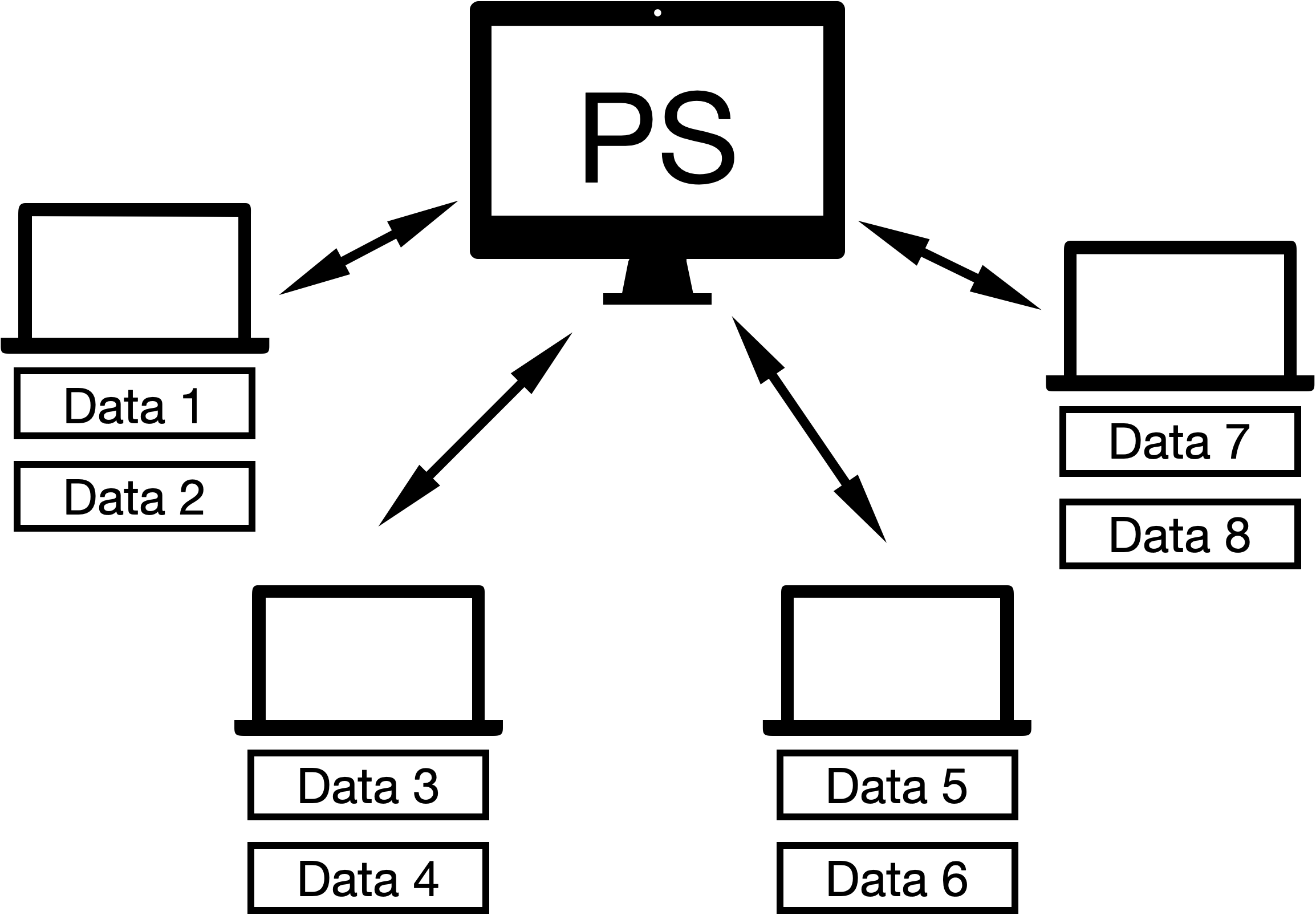} &   
  \includegraphics[width=4.5cm]{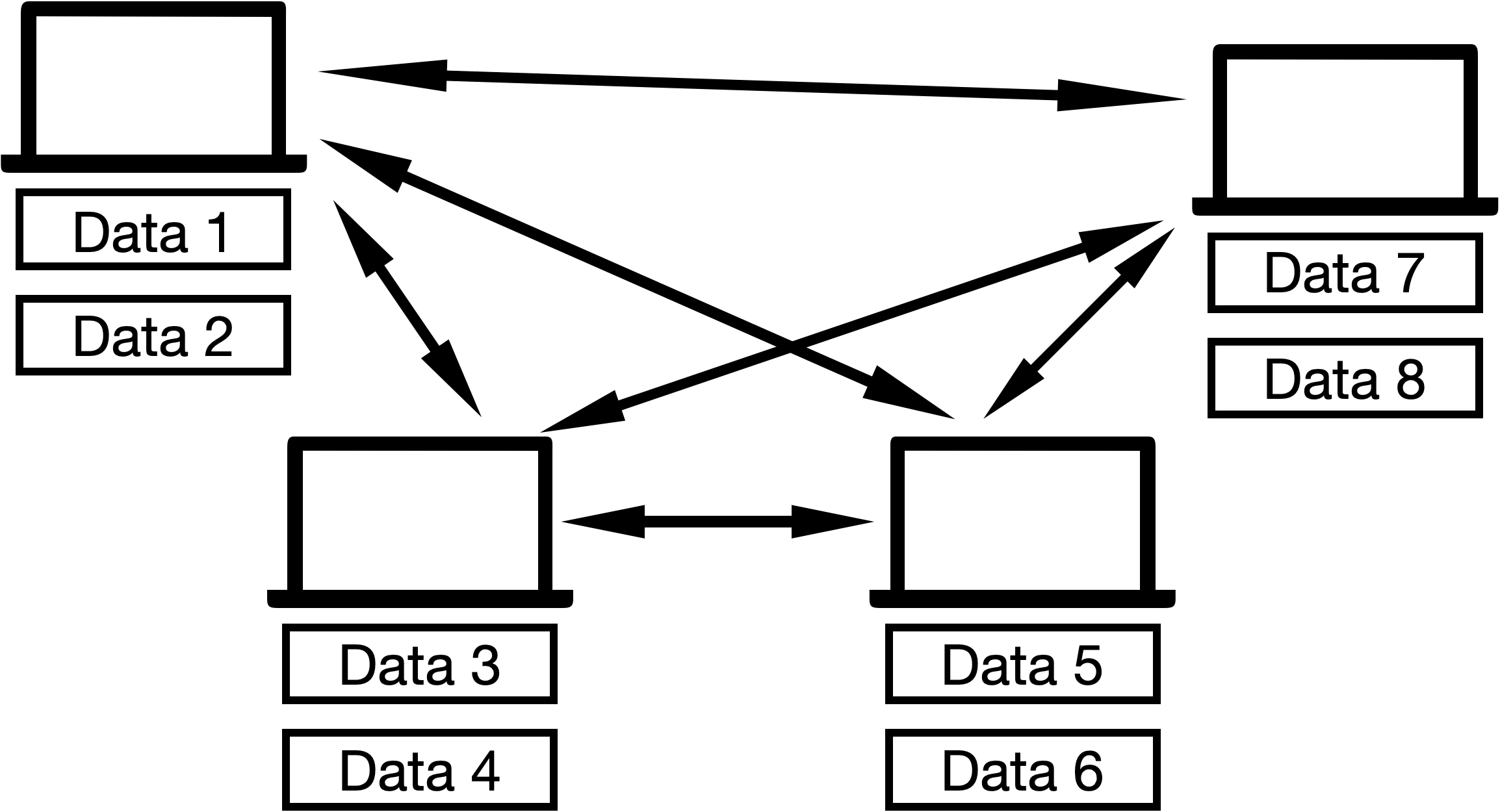}\\
(a) Shared Data, Central PS  & (b) Distributed Data, Central PS & (c) Distributed Data, Decentralized
\end{tabular}
\caption{ Settings for parallel optimization. 
(a) Shared-data setting, where ASAGA pertains~\citep{asaga_journal}.  (b) Distributed-data setting, where our work (ADSAGA) pertains.  (c) Decentralized setting, in which there are a variety of algorithms with weaker guarantees.}
\label{fig:settings}
\end{figure*}

In this paper, we focus on asynchronous algorithms for the distributed-data setting, under a stochastic delay model.
We consider the finite sum minimization problem common in many empirical risk minimization (ERM) problems:
\begin{equation}\label{eq:finite_sum}
\min_{x \in \mathbb{R}^d} f(x) := \frac{1}{n}\sum_{i = 1}^{n}{f_i(x)}, 
\end{equation} where each $f_i$ is convex and $\Lsup$-smooth, and $f$ is $\Lmain$-smooth and $\mu$-strongly convex. In machine learning, each $f_i$ represents a loss function evaluated at a data point. A typical strategy for minimizing finite sums is using \em variance-reduced \em stochastic gradient algorithms, such as SAG~\citep{sag}, SVRG~\citep{svrg} or SAGA~\citep{SAGA}. To converge to an $\epsilon$-approximate minimizer $x$, (that is, some $x$ such that $f(x) - \min_{x'}{f(x')} \leq \epsilon$), variance-reduced algorithms require $\tilde{O}(\left(n + \Lsup/\mu \right)\log(1/\epsilon))$ iterations. In contrast, the standard stochastic gradient descent (SGD) algorithm yields a slower convergence rate that scales with $1/\epsilon$.

Many of these algorithms can be distributed across $m$ machines who compute gradients updates \em asynchronously. \em
In such implementations, a gradient step is performed on the $k$'th central iterate $x^k$ as soon as a single machine completes its computation. Hence, the gradient updates performed at the PS may come from delayed gradients computed at stale copies of the parameter $x$. We denote this stale copy by $x^{k - \tau(k)}$, meaning that it is $\tau(k)$ iterations old. That is, at iteration $k$, the PS performs the update
\begin{equation*}
x^{k + 1} = x^k - \eta U(i_k, x^{k - \tau(k)}),
\end{equation*}
where $\eta$ is the learning rate, and $U(i_k, x^{k - \tau(k)})$ is an update computed from the gradient $\nabla f_{i_k}(x^{k - \tau(k)})$. For example, in SGD, we would have $U(i, x) = \nabla f_{i}(x)$. In the SAGA algorithm --- which forms the backbone of the algorithm we propose and analyze in this paper --- the update used is
\begin{equation}\label{eq:saga_update}
U(i, x) := \nabla f_i(x) - \alpha_i + \overline{\alpha},
\end{equation} where $\alpha_i$ is the prior gradient computed of $f_i$, and $\Abar$ denotes the average $\frac{1}{n}\sum_i\alpha_i$.\footnote{This update is variance-reduced because is an unbiased estimator of the gradient $\nabla f(x)$, and unlike the SGD update, its variance tends to $0$ as $x$ approaches the optimum of the objective (\ref{eq:finite_sum}).}

There has been a great deal of work analyzing asynchronous algorithms in settings where the data is \emph{shared} (or ``i.i.d.''), where any machine can access any of the data at any time, as in Figure~\ref{fig:settings}(a).
In particular, the asynchronous implementation of SAGA with shared data,
called ASAGA~\citep{asaga_journal} is shown to converge in $\tilde{O}((n + \tau_{max} L/\mu)\log(1/\epsilon))$ iterations, under arbitrary delays that are bounded by $\tau_{max}$. 
Other variance-reduced algorithms obtain similar results~\citep{mania2015perturbed, async_svrg, async_vr, mig}.

A key point in the analysis of asynchronous algorithms in the shared-data setting is the independence of the delay $\tau(k)$ at step $k$ and the function $f_{i_k}$ that is chosen at time $k$. This leads to an \emph{unbiased gradient condition}, namely that the expected update is proportional to $\nabla f(x^{k-\tau(k)})$.  This condition is central to the analyses of these algorithms. However, in the \em distributed-data \em (or ``non-i.i.d.") setting, where each machine only has access to the partition of data is stores locally (Figure~\ref{fig:settings}(b)), this condition does not naturally hold. For example, if the only assumption on the delays is that they are bounded, then using the standard SGD update $U(i, x) = \nabla f_i(x)$ may not even yield asymptotic convergence to $x^*$.

Due to this difficulty,
the asynchronous landscape is far less understood when the data is distributed.
Several works \citep{incremental, inc_prox_2, incremental_proximal} analyze an asynchronous incremental aggregation gradient (IAG) algorithm, which can be applied to the distributed setting; those works prove that with arbitrary delays, IAG converges deterministically in $\tilde{O}(\frac{n^2L}{\mu}\log(1/\epsilon))$ iterations, significantly slower than the results available for the shared data model. IAG uses the SAG update $U(i, x) = \nabla_i(x) + \sum_{j \neq i}\alpha_i$, where $\alpha_i$ is the prior gradient computed of $f_i$. The work \cite{fed_async} studies an asynchronous setting with distributed data and arbitrary delays and achieves an iteration complexity scaling polynomially with $1/\epsilon$. A line of work that considers a completely decentralized architecture without a PS (see Figure~\ref{fig:settings}(c)) generalizes the distributed data PS setting of Figure~\ref{fig:settings}(b). In this regime, with \em stochastic \em delays, \cite{xiangru} established sublinear convergence rates of $\tilde{O}(\Lsup^2/\epsilon^2)$ using an SGD update; however variance-reduced algorithms, which could yield linear convergence rates (scaling with $\log(1/\epsilon)$) for finite sums, have not been studied in this setting.

Because of the challenges distributed data --- in particular, the lack of independence between the data and the delays --- we adopt the stochastic delay model from \cite{xiangru} from the decentralized setting (Figure~\ref{fig:settings}(c)), rather than focusing on the worst-case delays that are commonly studied in the shared-data setting. 

We introduce the ADSAGA algorithm, a variant of SAGA designed for the distributed data setting. In a stochastic delay model (formalized below in Section~\ref{sec:model}), we show that ADSAGA converges in \\ $\tilde{O}\left((n + \Lsup/\mu + \sqrt{m\Lmain\Lsup}/\mu)\log(1/\epsilon)\right)$ iterations. To the best of our knowledge, this is the first provable result for asynchronous algorithms in the distributed-data setting --- under any delay model --- that scales both logarithmically in $1/\epsilon$ and linearly in $n$.
 Moreover, our empirical results suggest that ADSAGA in the distributed-data setting performs as well or better as SAGA in the shared-data setting.
 %\mkw{Added the ``Moreover'' thing here} \mg{Alternatively:}
 
 %\mg{TODO} \mkw{PLACEHOLDER TEXT FOR SPACE \lipsum[1]}
% that ADSAGA performs nearly as well in the distributed-data setting as ASAGA does in the shared-data setting, and outperforms other algorithms in the distributed-data setting, even when the experimental delay distribution differs from our model.  

\subsection{Our Model and Assumptions}\label{sec:model}
In this section, we describe our formal model.
\paragraph{Assumptions on the functions $f_i$.}  We study the finite-sum minimization problem 
\eqref{eq:finite_sum}.  As is standard in 
the literature on synchronous finite sum minimization (e.g. \citep{SAGA, asaga_journal, minibatch_saga, needell_SGD}), we assume that the functions $f_i$ are convex and $\Lsup$-smooth, that is, that
$$|\nabla f_i(x) - \nabla f_i (y) |_2^2 \leq \Lsup|x - y|_2^2 \qquad \forall x, y, i,$$ 
and we similarly assume that the objective $f$ is $\Lmain$-smooth.  We further assume that $f$ is $\mu$-strongly convex, that is, that
$$\langle{\nabla f(x) - \nabla f(y) ,x - y}\rangle \geq \mu|x - y|_2^2 \qquad \forall x, y.$$ 
Note that because $f$ is an average of the $f_i$, we have $\Lmain \leq \Lsup$.

\paragraph{Distribution of the data.}
We assume the distributed-data model in Figure~\ref{fig:settings}(b).
The functions $f_1, \ldots, f_n$ are partitioned among the $m$ machines into sets $\{S_j\}_{j \in [m]}$, such that each machine $j$ has access to $f_i$ for $i \in S_j$.\footnote{We assume that all sets $S_j$ have the same size, though if $m$ does not divide $n$, we can reduce $n$ until this is the case by combining pairs of functions to become one function.} 

\paragraph{Communication and delay model.}
The $m$ machines are governed by a centralized PS.  At timestep $k$, the PS holds an iterate $x^k$.  We consider the following model for asynchronous interaction.  
Let $\mathcal{P} = (p_1, \ldots, p_m)$ denote a probability distribution on the $m$ machines.
Each machine $j$ holds a (possibly stale) iterate $x_j$.
At timestep $k$, a random machine $j$ is chosen with probability $p_j$.  This machine $j$ sends an update $h_j$ to the PS, based on $x_j$ and the data it holds (that is, the functions $f_i$ for $i \in S_j$).  The PS sends machine $j$ the current iterate, and machine $j$ updates $x_j \gets x^k$.  Then the PS performs an update based on $h_j$ to obtain $x^{k+1}$, the iterate for step $k+1$.  Then the process repeats, and a new machine is chosen independently from the distribution $(p_1, \ldots, p_m)$.

\begin{remark}[Relationship to prior delay models]\label{rem:exponential}
Random delay models, such as the one we describe above, have been studied more generally in decentralized asynchronous settings \citep{xiangru, randomized_gossip, gossip_complete} where they are often referred to as ``random gossip". In a random gossip model, each machine has an exponentially distributed clock and wakes up to communicate an update with its neighbors each time it ticks.  

Our model is essentially the same as the random gossip model, restricted to a centralized communication graph, as in \cite{gossip_complete}. That is, our discrete delay model arises from a continuous-time model where each machine $j$ takes $T_j$ time to compute its update, where $T_j$ is a random variable distributed according to an exponential distribution with parameter $\lambda_j$.  After $T_j$ time, machine $j$ sends its update to the PS and receives an updated iterate $x_j \gets x^k$.  Then it draws a new (independent) work time $T_j$ and repeats.  All of the machines have independent work times, although they may have different rates $\lambda_j$. 

It is not hard to see that, due to the memorylessness of exponential random variables, this continuous-time model is equivalent to the discrete-time model described above.  

Similarly, our model generalizes the geometric delay model in \citep{begets} for centralized asynchronous gradient descent, which is the special case where the $\lambda_j$'s are all the same.
\end{remark}

We note that, in our model, the PS has knowledge of the values $p_i$, and we use this in our algorithm.  In practice, these rates can be estimated from the ratio between the number of updates from the machine $j$ and the total number of iterations.

Stochastic delays are well-motivated by applications. In the data-center setting, stochasticity in machine performance --- and in particular, heavy tails in compute times --- is well-documented~\citep{dean2013tail}. Because we allow for the $p_i$'s to be distinct for each machine, our model is well-suited to the example of federated learning, where we expect machines to have heterogeneous delays. Our particular model for stochastic delays is natural because it fits into the framework of randomized gossip and arises from independent exponentially distributed work times; however, we believe a more general model for stochastic delays could be more practical and merits further study.

\begin{remark}[Blocking data]\label{remark:blocking}
In our model, a machine sends an update computed from a single $f_i$ in each round, which is impractical when communication is expensive. In a practical implementation, we could block the data into blocks $B_\ell$ of size $b$, such that each machine computes $b$ gradients before communicating with the PS. To apply our result, the set of functions $\{f_{i}\}_i$ is replaced with $\{\sum_{i \in B_\ell} f_i\}_{\ell}$, yielding an iteration complexity of $\tilde{O}\left(n + b\Lsup/\mu + b\sqrt{m\Lmain\Lsup / \mu}\right)$.
\end{remark}

\begin{table*}
\begin{center}
\caption{Comparison of related work for minimization of finite sums of $n$ convex and $\Lsup$-smooth functions, whose average is $\mu$-strongly convex and $\Lmain$-smooth. Here $m$ denotes the number of machines or the minibatch size. We have substituted $O(m)$ for the maximum overlap bound $\tau$  in \citet{asaga_journal, mig}, which is at least $m$, and $O(n)$ for the delay bound in IAG, which is as least $n$. We note that, as per Remark~\ref{rem:exponential}, our stochastic model is essentially a restriction of randomized gossip to a centralized communication graph.}
\vspace{.3cm}
\begin{tabular}{|p{2cm}|p{2.2cm}|p{2.8cm}|p{3.4cm}|p{5.1cm}|}
\hline
     Algorithm & Data Location & Delay Model & Convergence Guarantee & Iteration Complexity \\ 
     \hline 
     \textbf{This work} (ADSAGA, Theorem~\ref{thm:main}) & Distributed & Stochastic, $p_j = \Theta(\frac{1}{m}) \:\: \forall j$ (See Section~\ref{sec:model}) & $\mathbb{E}\left[f(x^k) - f(x^*)\right]\leq \epsilon$ & $\tilde{O}\left(\left(n + \frac{\Lsup}{\mu} + \frac{\sqrt{m\Lmain\Lsup}}{\mu}\right)\log(1/\epsilon)\right)$
\\
     \hline
     ASAGA, MiG & Shared & Arbitrary, bounded & $\mathbb{E}\left[f(x^k) - f(x^*)\right] \leq \epsilon$ & $O\left((n + \frac{m\Lsup}{\mu})\log(1/\epsilon)\right)$ \\
     \hline
     IAG & Distributed & Arbitrary, bounded & $f(x^k) - f(x^*) \leq \epsilon$ & $O\left(\frac{n^2\Lsup}{\mu}\log(1/\epsilon)\right)$  \\
     \hline
     Decentralized SGD \citep{xiangru} & Distributed & Random Gossip (see Remark~\ref{rem:exponential}) & $\mathbb{E}\left[f(x^k) - f(x^*)\right] \leq \epsilon$  & $\tilde{O}(\frac{L^2}{\epsilon^2})$\\
     \hline
    FedAsync~\citep{fed_async} & Distributed & Arbitrary, bounded & $\mathbb{E}\left[f(x^k) - f(x^*)\right] \leq \epsilon$  & $\tilde{O}(\frac{L^2}{\epsilon^2})$\\
     \hline
     Minibatch SAGA (Prop. \ref{prop:minibatch}) & Shared or Distributed & Synchronous (No Delays)& $\mathbb{E}|x^k - x^*|_2^2 \leq \epsilon$ & $\tilde{O}\left(\left(n + \frac{\Lsup}{\mu} + \frac{m\Lmain}{\mu}\right)\log(1/\epsilon)\right)$\\ 
     \hline
\end{tabular}
\label{table:related_work}
\end{center}
\end{table*}

\subsection{Contributions}
Our main technical contribution is the development and analysis of a SAGA-like algorithm, which we call ADSAGA, in the model described in Section~\ref{sec:model}. We show that with $m$ machines, for $\mu$-strongly convex, $\Lmain$-smooth functions $f$ with minimizer $x^*$, when each $f_i$ is convex and $\Lsup$-smooth, ADSAGA achieves $|f(x_k) - f(x^*)|_2^2 \leq \epsilon$ in
\begin{equation}\label{rate}
k = \tilde{O}\left(\frac{1}{m\pmin}\left(n + \frac{\Lsup}{\mu} + \frac{\sqrt{m\Lmain\Lsup}}{\mu}\right)\log(1/\epsilon)\right)
\end{equation} 
iterations, where $\pmin = \min(p_j)$ is the minimum update rate parameter. Standard sequential SAGA achieves the same convergence in $O\left((n + \frac{\Lsup}{\mu})\log(1/\epsilon)\right)$ iterations. This implies that when the machine update rates vary by no more than a constant factor, if the term $n + \Lsup/\mu$ in our iteration complexity \eqref{rate} dominates the final term $\sqrt{m\Lmain\Lsup/\mu}$, the ADSAGA with stochastic delays achieves the same iteration complexity as SAGA with no delays. On the other hand, when the $\sqrt{m\Lmain\Lsup/\mu}$ term dominates in \eqref{rate}, the convergence rate of ADSAGA scales with the \textit{square root} of the number of machines, or average delay.  %\mg{can you read this over?}\mkw{looks good}

Remarkably, due to this $\sqrt{m}$ dependence, the convergence rate in \eqref{rate} is dramatically \em faster \em than the rates proved for ASAGA (the analog in the shared-data setting with arbitrary delays), which as discussed above is at best $O(n + \frac{Lm}{\mu})$, or even the rate of synchronous minibatch SAGA, which is $\tilde{O}(n + \frac{L}{\mu} + \frac{mL_f}{\mu})$. 
(Note that, in ASAGA, the convergence rate scales with the maximum delay $\tau_{max}$, which is lower bounded by $m$.)
The fast convergence rate we give for ADSAGA is possible because our stochastic delay model. Due to the occasional occurrence of very short delays (e.g., a machine $j$ is drawn twice in quick succession), after the same number of iterations, the parallel depth of ADSAGA in our model is far deeper than that of a synchronous minibatch algorithm or than some instantiations of ASAGA with bounded delays. 

The proof of our result uses a novel potential function to track our progress towards the optimum. In addition to including typical terms such as $f(x_k) - f(x^*)$ and $|x_k - x^*|_2^2$, our potential function includes a quadratic term that takes into account the dot product of $x_k - x^*$ and the expected next stale gradient update. This quadratic term is similar to the one that appears in the potential analysis of SAG. Key to our analysis is a new \em unbiased trajectory \em lemma, which states that in expectation, the expected stale update moves towards the true gradient.

We support our theoretical claims with numerical experiments.
In our first set of experiments, we simulate ADSAGA as well as other algorithms in the model of Section~\ref{sec:model}, and show that ADSAGA not only achieves better iteration complexity than other algorithms that have been analyzed in the distributed-data setting (IAG, SGD), but also does nearly as well, in the distributed-data setting, as SAGA does in the shared-data setting.  In our second set of experiments, we implement ADSAGA and other state-of-the-art distributed algorithms in a distributed compute cluster. We observe that in terms of wallclock time, ADSAGA performs similarly to IAG, and both of these asynchronous algorithms are over 60\% faster with 30 machines than the synchronous alternatives (minibatch-SAGA and minibatch-SGD) along with asynchronous SGD.

% Moreover, our empirical results demonstrate that in the stochastic delay model of Section~\ref{sec:model}, ADSAGA achieves a faster iteration complexity than the other algorithms, IAG and SGD, that have been analyzed in a distributed-data setting. In experiments on a distributed compute cluster, we observe that all the algorithms in a distributed-data setting achieve wall-clock advantages over ASAGA with shared data.

%which show how the convergence rate changes both when the work times deviate from the model in Section~\ref{sec:model}, and when the number of machines changes. We show empirically that \mkw{TODO. \lipsum[1]}
%in linear regression, where frequently $\Lmain \ll \Lsup$, we achieve parallel speedups from increasing $m$ as large as $n$. Our empirical results --- based on i.i.d. work times from shifted exponential distributions --- suggest that {ADSAGA}
%%, our adaptation of SAGA to a distributed asynchronous setting, 
%performs nearly as well as ASAGA, and better than previously proposed algorithms for the distributed setting such as IAG and SGD.

\begin{remark}[Extensions to non-strongly convex objectives, and acceleration]
We remark that the AdaptReg reduction in \citet{reductions} can be applied to ADSAGA to extend our result to non-strongly convex objectives $f$. In this case, the convergence rate becomes $\tilde{O}(n + \Lsup/\epsilon + \sqrt{m\Lmain\Lsup}/\epsilon)$. Further, applying the black-box acceleration reduction in \citet{catalyst} or \citet{a_prox} yields a convergence rate of $\tilde{O}(n + \sqrt{n\Lsup/\mu} + \sqrt{n\sqrt{m\Lmain\Lsup}/\mu})$. Both of the reductions use an outer loop around ADSAGA which requires breaking the asynchrony to serialize every $\Omega(n)$ iterations. We expect that using the outer loops without serializing would yield the same results.
\end{remark}
\subsection{Related Work}
We survey the most related gradient-based asynchronous algorithms for strongly convex optimization, focusing on results for the setting of finite sums. For completeness, we state some results that apply to the more general setting of optimization over (possibly non-finite) data distributions. See Table~\ref{table:related_work} for a quantitative summary of the most relevant other works.

\textbf{Synchronous Parallel Stochastic Algorithms.}
Synchronous parallel stochastic gradient descent algorithms can be thought of as minibatch variants of their non-parallel counterparts. Minibatch SGD is analyzed in \citet{psgd, batch_sgd}. For finite sum minimization, minibatch SAGA is analyzed in \citet{minibatch_saga, mb_saga2, mb_saga3, minibatch_lb}, achieving a convergence rate of $O\left(\left(n + \frac{\Lsup}{\mu} + \frac{m\Lmain}{\mu}\right)\log(1/\epsilon)\right)$ for a minibatch size of $m$ with distributed data.\footnote{We prove this result in Proposition~\ref{prop:minibatch} of the appendix, as the cited works \citep{minibatch_saga, mb_saga2, mb_saga3} prove slightly weaker bounds for minibatch SAGA using different condition numbers.} Distributed SVRG is analyzed in \citet{dsvrg}. Katyusha \citep{Katyusha} presents an accelerated, variance reduced parallelizable algorithm for finite sums with convergence rate $O\left(\left(n + \sqrt{\frac{n\Lsup}{\mu}} + m\sqrt{\frac{\Lmain}{\mu}}\right)\log(1/\epsilon)\right)$; this rate is proved to be near-optimal in \citet{minibatch_lb}. \citet{hetero_FL} compares local mini-batching and local-SGD for the more general setting of non-i.i.d data.

\textbf{Asynchronous Centralized Algorithms with Shared Data  (Figure~\ref{fig:settings}(a))}
Centralized asynchronous algorithms often arise in shared-memory architectures or in compute systems with a central parameter server. The textbook \citep{textbook} shows asymptotic convergence for stochastic optimization in \textit{totally asynchronous} settings which may have unbounded delays. In the \textit{partially asynchronous} setting, where delays are arbitrary but bounded by some value $\tau$, sublinear convergence rates of $O(\frac{1}{\epsilon})$ matching those of SGD were achieved for strongly convex stochastic optimization in \citet{hogwild} (under sparsity assumptions) and \citet{duchi_sgd}. For finite sum minimization, linear convergence is proved for asynchronous variance-reduced algorithms in \citet{mania2015perturbed, async_svrg, async_vr, asaga_journal, mig}; the best known rate of $O\left(\left(n + \frac{\tau\Lsup}{\mu}\right)\log(1/\epsilon)\right)$ is achieved by ASAGA~\citep{asaga_journal} and MiG~\citep{mig}, though these works provide stronger guarantees under sparsity assumptions. Note that most of these works can be applied to lock-free shared-memory architectures as they do not assume consistent reads of the central parameter. \citet{fixed_delay} considers the setting where all delays are exactly equal to $\tau$.

\textbf{Asynchronous Centralized Algorithms with Distributed Data (Figure~\ref{fig:settings}(b))}
Several works \citep{incremental, incremental_proximal, inc_prox_2} consider incremental aggregated gradient (IAG) algorithms, which use the update $U(i, x) = \nabla f_i(x) + \sum_{i' \neq i}{\alpha_i}$, and can be applied to the distributed data setting. All of these works yield convergence rates that are quadratic in the maximum delay between computations of $\nabla f_i$. Note that this delay is lower bounded by $n$ if a single new gradient is computed at each iteration. The bounds in these works are deterministic, and hence cannot leverage any stochasticity in the gradient computed locally at each machine, which is natural when $n > m$ and each machine holds many functions $f_i$. \citet{fed_async} studies asynchronous federated optimization with arbitrary (bounded) delays; this work achieves a convergence rate that scales with $\frac{1}{\epsilon}$.

\textbf{Asynchronous Decentralized Algorithms with Distributed Data (Figure~\ref{fig:settings}(c))} 
In the decentralized setting, the network of machines is represented as a graph $G$, and machines communicate (``gossip") with their neighbors. Many works \citep{randomized_gossip, gossip_complete, xiangru} have considered the setting of \textit{randomized} gossip, where each machine has an exponentially distributed clock and wakes up to communicate with its neighbors each time it ticks. In \citet{xiangru}, a convergence rate of $O(1/\epsilon^2)$ is achieved for non-convex objectives $f$, matching the rate of SGD. We remark that by choosing the graph $G$ to be the complete graph, this result extends to our model. For strongly convex objectives $f$, \citet{decentralized_linear} studied a decentralized setting with arbitrary but bounded delays, and achieved a linear rate of convergence using a gradient tracking technique.

We refer the reader to \citet{survey} for a recent survey on asynchronous parallel optimization algorithms for a more complete discussion of compute architectures and asynchronous algorithms such as coordinate decent methods that are beyond the scope of this section.

\subsection{Organization}
The rest of the paper is organized as follows. In Section~\ref{sec:formal_setup} we formally set up the problem and the ADSAGA algorithm. In Section~\ref{sec:proofs} we state our main results and sketch the proof. In Section~\ref{sec:experiments} we provide empirical simulations. We conclude and discuss future directions in Section~\ref{sec:conclusion}. The Appendix contains proofs.

\subsection{Notation}
For symmetric matrices $A$ and $B$ in $\mathbb{R}^{d \times d}$, we use the notation $A \succeq B$ to mean $x^TAx \geq x^TBx$ for all vectors $x \in \mathbb{R}^d$. We use $I_d$ to denote the identity matrix in $\mathbb{R}^{d \times d}$, and $\mathbbm{1}$ to denote the all-ones vector. We use the notation $\tilde{O}$ to hide logarithmic factors in $n$ or in constants depending on the functions $f_i$.

\section{The ADSAGA Algorithm}\label{sec:formal_setup}

%\section{Formal Set-up and Algorithm}\label{sec:formal_setup}
% \mkw{moved these assms earlier}
%In this work, we study the following finite-sum minimization problem:
%\begin{equation}
%\min_{x \in \mathbb{R}^d} f(x) := \frac{1}{n}\sum_{i = 1}^{n}{f_i(x)}. 
%\end{equation}
%We make the following assumptions, which are standard in the literature on synchronous finite sum minimization (e.g. \citep{SAGA, asaga_journal, minibatch_saga, needell_SGD}).
%We assume that the functions $f_i$ are convex and $\Lsup$-smooth, that is, 
%$$|\nabla f_i(x) - \nabla f_i (y) |_2^2 \leq \Lsup|x - y|_2^2 \qquad \forall x, y, i.$$ We further assume that the objective $f$ is $\Lmain$-smooth and $\mu$-strongly convex, that is,
%$$\langle{\nabla f(x) - \nabla f(y) ,x - y}\rangle \geq \mu|x - y|_2^2 \qquad \forall x, y.$$ 
%Note that because $f$ is an average of the $f_i$, we have $\Lmain \leq \Lsup$.

%In ADSAGA, we partition the $n$ functions equally among the $m$ machines into sets $\{S_j\}_{j \in [m]}$, such that each machine $j$ has access to $f_i$ for for $i \in S_j$.\footnote{We assume that all sets $S_j$ have the same size, though if $m$ does not divide $n$, we can reduce $n$ until this is the case by combining pairs of functions to become one function.} 

In this section, we describe the algorithm ADSAGA, a variant of SAGA designed for an the asynchronous, distributed data setting. We work in the model described in Section~\ref{sec:model}.  Each machine maintains a local copy of the iterate $x$, which we denote $x_j$, and also stores a vector $\alpha_i$ for each $i \in S_j$, which contains the last gradient of $f_i$ computed at machine $j$ and sent to the PS. The central PS stores the current iterate $x$ and maintains the average $\overline{\alpha}$ of the $\alpha_i$. Additionally, to handle the case of heterogeneous update rates, the PS maintains a variable $u_j$ for each machine, which stores a weighted history of updates from machine $j$.

\begin{algorithm}[b!]
  \caption{Asynchronous Distributed SAGA (ADSAGA): Implementation}
  \label{local_stoc_data_implementation}
  \algrenewcommand\algorithmicprocedure{\textbf{process}}
\begin{algorithmic}
%%%% PARAMETER SERVER
\Procedure{Parameter Server}{$\eta, \{p_j\}, x^{(0)}, t$}
\State $u_j = 0$ for $j \in [m]$  \Comment{Initialize update variables}
\State $\overline{\alpha} =0$ \Comment{Initialize last gradient averages}

\Repeat \:
%\State \textbf{Communication:}
%\Indent
\State C.1) Receive $h_j$ from machine $j$
\State C.2) Send $x$ to machine $j$
%\EndIndent
%\State \textbf{Gradient Update at PS:}
%\Indent
\State PS.1) $u_j \leftarrow u_j\left(1 - \frac{\pmin}{p_j}\right) + h_j$ \Comment{First update to $u_j$}
\State PS.2) $x \leftarrow x - \eta_j\left(u_j + \overline{\alpha}\right)$ \Comment{Apply stale gradient update using step size $\eta_j = \frac{\eta \pmin}{p_j}$}
\State PS.3) $\overline{\alpha} \leftarrow \overline{\alpha} + \frac{1}{n}h_j$ \Comment{Update gradient averages}
\State PS.4) $u_j \leftarrow u_j - \frac{m}{n}h_j$ \Comment{Second update to $u_j$}
%\EndIndent
\Until{$t$ total updates have been made} \\
\Return{$x$}
\EndProcedure

\vspace{.5cm}

%%%% WORKER
\Procedure{Worker Machine}{$\{f_i\}_{i \in S_j}, x_j$}
\State $h_j = 0$ for $j \in [m]$ \Comment{Initialize updates to $0$ at each machine}
\State $\alpha_i =0$ for $i \in [n]$\Comment{Initialize last gradients to 0 at each machine}
\Repeat \: %\textbf{until terminated by PS}
%\State \textbf{Communication:}
%\Indent
\State C.1) Send $h_j$ to PS %\Comment{Send $h_j$ from machine $j$ to PS}
\State C.2) Receive current iterate $x$ from PS and set $x_j \gets x$ % \Comment{Receive current iterate $x$ from PS}
%\EndIndent
%\State \textbf{Locally at machine j:}
%\Indent
\State  M.1) $i \sim \text{Uniform}(S_j)$ \Comment{Choose a random function}
\State M.2) $h_j \leftarrow \nabla f_{i}(x_j) - \alpha_i$ \Comment{Evaluate gradient $\nabla f_i$ at local iterate minus last time's $\nabla f_i$}
\State M.3) $\alpha_{i} \leftarrow \nabla f_{i}(x_j)$ \Comment{Update last gradient locally}
%\EndIndent
\Until{terminated by PS}
\EndProcedure
\end{algorithmic}
\end{algorithm}

We describe the algorithm formally in Algorithm~\ref{local_stoc_data_implementation}. At a high level, each machine $j$ chooses a function $f_i$ randomly in $S_j$, and computes the variance-reduced stochastic gradient $h_j := \nabla f_i(x_j) - \alpha_i$, that is, the gradient of $f_i$ at the current local iterate minus the prior gradient this machine computed for $f_i$. Meanwhile, upon receiving this vector $h_j$ at the PS, the PS takes a gradient step in roughly the direction $h_j + \Abar$. In the case where the machines' update rates $p_j$ are equal, $u_j = h_j$ always, so our algorithm is precisely an asynchronous implementation of the SAGA algorithm with delays (see eq.~\ref{eq:saga_update}).

If the machine update rates are heterogeneous, our algorithm differs in two ways from SAGA. First, we use machine-specific step sizes $\eta_j$ which scale inversely with the machine's update rate, $p_j$.\footnote{This requires that the PS has knowledge of these rates; in practice, these rates can be estimated from the ratio between the number of updates from machine $j$ and the total number of iterations.} Intuitively, this compensates for less frequent updates with larger weights for those updates. Second, we use the variables $u_j$ so that the trajectory of the expected update to $x$ tends towards the full gradient $\nabla f(x)$ (see Lemma~\ref{unbiased}).

We provide a logical view of ADSAGA (in the delay model described in Section~\ref{sec:model}) in Algorithm~\ref{local_stoc_data}. 
We emphasize that Algorithm~\ref{local_stoc_data_implementation} is equivalent to Algorithm~\ref{local_stoc_data} given our stochastic delay model; we introduce Algorithm~\ref{local_stoc_data} only to aid the analysis. In the logical view, each logical iteration tracks the steps performed by the PS when $h_j$ is sent to the PS from some machine $j$, followed by the steps performed by the machine $j$ upon receiving the iterate $x$ in return. We choose this sequence for a logical iteration because it implies that iterate $x_j$ used to compute the local gradient in step M.2 equals the iterate $x$ from the PS. Because the variable $u_j$ is only modified in iterations which concern machine $j$, we are able to move step PS.2 to later in the logical iteration; this eases the analysis. For similar reasons, we move the step M.3 which updates $\alpha_i$ to the start of the logical iteration.

To make notation clearer for the analysis, in Algorithm~\ref{local_stoc_data}, we index the central parameter with a superscript of the iteration counter $k$. Note that in this algorithm, we also introduce the auxiliary variables $g_j, \beta_j$ and $i_j$ to aid with the analysis. The variable $i_j$ tracks the index of the function used to compute the update $h_j$ at machine $j$, $g_j := \nabla_i(x_j)$, and $\beta_j := \alpha_{i_j}$. 

\begin{algorithm}[t]
  \caption{Asynchronous Distributed SAGA (ADSAGA): Logical View, in the model in Section~\ref{sec:model}} % with Exponential Work Times}
  \label{local_stoc_data}
\begin{algorithmic}
\Procedure{ADSAGA}{ $x^0, \eta, \{f_i\}, \{S_j\}, \mathcal{P}, t$}
\State $g_j, h_j, u_j, \beta_j = 0$ for $j \in [m]$ \Comment{Initialize variables to $0$}
\State $i_j \sim\text{Uniform}(S_j)$ for $j \in [m]$
\Comment{Randomly initialize last gradient indicator at each machine}
\State $\alpha_i =0$ for $i \in [n]$\Comment{Initialize last gradients to 0 at each machine}
\State $\overline{\alpha} =0$ \Comment{Initialize last gradient averages at PS}
\For{$k=0$ {\bfseries to} $t$}
\State  $j \sim \mathcal{P}$; \Comment{Choose machine $j$ to wake up with probability $p_j$}
\State M.3) $\alpha_{i_j} \leftarrow g_j$  \Comment{Update last gradient locally}
% \State C.2) $x_j \leftarrow x^k$ \Comment{Read current iterated from PS}
\State PS.2) $x^{k+1} \leftarrow x^k - \eta_j\left(u_j + \overline{\alpha}\right)$
\Comment{Take gradient step using $\eta_j = \frac{\eta \pmin}{p_j}$}
\State PS.3) $\overline{\alpha} \leftarrow \overline{\alpha} + \frac{1}{n}h_j$ \Comment{Update gradient averages at PS}
\State PS.4) $u_j \leftarrow u_j - \frac{m}{n}h_j$ \Comment{First update to $u_j$}
\State M.1) $i \sim \text{Uniform}(S_j)$ \Comment{Choose uniformly a random function at machine $j$}
\State $g_j \leftarrow \nabla f_{i}(x^k)$  \Comment{Update auxiliary variable $g_j$}
\State $\beta_j \leftarrow \alpha_i$ \Comment{Update auxiliary variable $\beta_j$}
\State M.2)  $h_j \leftarrow \nabla f_{i}(x_j) - \alpha_i$ \Comment{Prepare next update to be sent to PS}
\State PS.1) $u_j \leftarrow u_j\left(1 - \frac{\pmin}{p_j}\right) + \frac{\pmin}{p_j} h_j$ \Comment{Second update of $u_j$ at PS; This could also be done locally}
\State $i_j \leftarrow i$ \Comment{Update auxiliary variable $i_j$}
\EndFor\\
\Return{$x^t$}
\EndProcedure
\end{algorithmic}
\end{algorithm}

\section{Convergence Result and Proof Overview}\label{sec:proofs}

In this section, we state and sketch the proof of our main result, which yields a convergence rate of $\rate$ when $p_j = \Theta\left(\frac{1}{m}\right)$ for all $j$, that is, all machines perform updates with the same frequecy up to a constant factor.

\begin{restatable}{theorem}{mainthm}\label{thm:main}
Let $f(x) = \frac{1}{n}\sum_{i = 1}^n{f_i(x)}$ be an $\Lmain$-smooth and $\mu$-strongly convex function. Suppose that each $f_i$ is $\Lsup$-smooth and convex. Let $r:=  \frac{8\left(76 + 168\left(\frac{\pmax}{\pmin}\right)^2\frac{m}{n}\right)}{3}$. For any partition of the $n$ functions to $m$ machines, after $$k = m\pmin\left(4n + 2r\frac{\Lsup}{\mu} + 2\sqrt{r}\frac{\sqrt{m\Lmain\Lsup}}{\mu}\right)\log\left(\frac{\left(1 + \frac{1}{2m\mu\eta}\right)\left(f(x^0) - f(x^*)\right) + \frac{n\sigma^2}{2\Lsup}}{\epsilon}\right)$$ iterations of Algorithm~\ref{local_stoc_data} with $\eta = \etabd$,
we have $\mathbb{E}\left[f(x^k) - f(x^*)\right] \leq \epsilon,$ where $\sigma^2 = \frac{1}{n}\sum_i|\nabla f_i(x^*)|_2^2$.
\end{restatable}

We prove this theorem in Section~\ref{sec:main_proof}. The key elements of our proof are the \textit{Unbiased Trajectory Lemma} (Lemma~\ref{unbiased}) and a novel potential function which captures progress both in the iterate $x_k$ and in the stale gradients. Throughout, all expectations are over the choice of $j \sim \mathcal{P}$ and $i \sim \text{Uniform}(S_j)$ in each iteration of the logical algorithm. 

We begin by introducing some notation which will be used in defining the potential function. Let $H$, $G$, and $U$ be the matrices whose $j$th columns contain the vectors $h_j = g_j - \beta_j$,  $g_j$, and $u_j$ respectively. We use the superscript $k$ to denote the value of any variable from Algorithm~\ref{local_stoc_data} at the \textit{beginning} of iteration $k$. When the iteration $k$ is clear from context, we will eliminate the superscripts $k$. To further simplify, we will use the following definitions: $\Astar_i := \alpha_i - \nabla f_i(x^*)$, $\beta^*_j := \beta_j - \nabla f_{i_j}(x^*)$, and $g^*_j := g_j - \nabla f_{i_j}(x^*)$, where $i_j$ is the index of the function used by machine $j$ to compute $g_j$ and $\beta_j$, as indicated in Algorithm~\ref{local_stoc_data}.

We will analyze the expectation of the following potential function $\phi(x, G, H, U, \alpha, \beta)$:
$$\phi(x, G, H, U, \alpha, \beta) := \sum_{\ell}\phi_{\ell}(x, G, H, U, \alpha, \beta),$$ where
\begin{equation*}
\begin{split}
\phi_1(x, G, H, U, \alpha, \beta) &:= 4m\eta\left(f(x) - f(x^*)\right), \\
\phi_2(x, G, H, U, \alpha, \beta) &:= \begin{pmatrix}x - x^*\\ \eta (U\mathbbm{1} + m\overline{\alpha})\end{pmatrix}^T\begin{pmatrix}I_d & -I_d \\ -I_d & 2I_d\end{pmatrix}\begin{pmatrix}x - x^* \\ \eta(U\mathbbm{1} + m\overline{\alpha})\end{pmatrix}, \\
\phi_3(x, G, H, U, \alpha, \beta) &:= \eta^2c_3\sum_j{\frac{\pmin}{p_j}|g_j^*|_2^2}, \\
\phi_4(x, G, H, U, \alpha, \beta) &:= \eta^2c_4\left(2\sum_i{\frac{\pmin}{p_j}|\Astar_i|_2^2} - \sum_j{\frac{\pmin}{p_j}|\beta_j^*|_2^2}\right), \\
\phi_5(x, G, H, U, \alpha, \beta) &:= \eta^2c_5\sum_j{|u_j|_2^2},
\end{split}
\end{equation*}
and
\begin{equation*}
\begin{split}
c_5 &:= \frac{16}{3}(m\Lmain\eta + 1),\\
c_3, c_4 &= \Theta\left(1 + \frac{m}{n}\left(\frac{\pmax}{\pmin}\right)^2\right)c_5.
\end{split}
\end{equation*}
The exact values of $c_3$ and $c_4$ are given in the Appendix.

It is easy to check that the potential function is non-negative. In particular, $\phi_1$ and $\phi_3$, $\phi_5$ are clearly non-negative, and $\phi_2$ is non-negative because $\bigl( \begin{smallmatrix}1 & -1\\ -1 & 2\end{smallmatrix}\bigr)$ is positive definite.  Finally, $\phi_4$ is non-negative because the terms in the sum over $j$ are a subset of the terms in the sum over $i$.

This potential function captures not only progress in $f(x^k) - f(x^*)$ and $|x^k - x^*|_2^2$, but also the extent to which the expected stale update, $\frac{1}{m}U\mathbbm{1} + \Abar$, is oriented in the direction of $x^k - x^*$. While some steps of asynchronous gradient descent may take us in expectation further from the optimum, those steps will position us for later progress by better orienting $\frac{1}{m}U\mathbbm{1} + \Abar$. The exact coefficients in the potential function are chosen to cancel extraneous quantities that arise when evaluating the expected difference in potential between steps. In the rest of the text, we abbreviate the potential $\phi(x^k, G^k, H^k, U^k, \alpha^k, \beta^k)$, given by the variables at the start of the $k$th iteration, by $\phi(k)$. All expectations below are over the random choices of $j \sim \mathcal{P}$ and $i \sim \text{Uniform}(S_j)$ in the $k$th iteration of Algorithm~\ref{local_stoc_data}, and we implicitly condition on the history $\{x^k, G^k, H^k, U^k, \alpha^k, \beta^k\}$ in such expectations.

The following proposition is our main technical proposition.
\begin{restatable}{proposition}{mainprop}\label{single_step} In Algorithm~\ref{local_stoc_data}, for any step size $\eta \leq \etabd$, 
$$\mathbb{E}_{i, j}[\phi(k + 1)] \leq \left(1 - \gamma\right)\phi(k)$$

where $\gamma = m\pmin\min\left(\frac{1}{4n}, \mu\eta\right)$, and $r$ is a constant defined in Theorem~\ref{thm:main} dependent only on $\frac{\pmax}{\pmin}$.
\end{restatable}

We sketch the proof of this proposition. For ease of presentation, we assume in this section that $p_j = \Theta\left(\frac{1}{m}\right)$ for all $j$. The formal proof, which contains precise constants and the dependence on $\pmin$, is given in Section~\ref{sec:main_proof}. 
We begin by stating the Unbiased Trajectory lemma, which shows that the expected update to $x$ moves in expectation towards the gradient $\nabla f(x)$.

\begin{lemma}[Unbiased Trajectory]\label{unbiased}
At any iteration $k$, we have 
\begin{equation*}
    \mathbb{E}_{i, j}[x^{k + 1}] = x^k - \eta\pmin \left(U^k\mathbbm{1} + m\Abar^k \right),
\end{equation*}
and 
\begin{equation*}
    \mathbb{E}_{i, j}\left[U^{k + 1}\mathbbm{1} +  m\Abar^{k + 1}\right] = \pmin\left(1 - \frac{1}{n}\right)\left(U^{k}\mathbbm{1} + m\Abar^{k}\right) + m \pmin  \nabla f(x^k).
\end{equation*}
\end{lemma}

Using this condition, we can control the expected change in $\phi_1 + \phi_2$, yielding the following lemma, stated with precise constants in Section~\ref{sec:main_proof}. Let $q:= 1 + m\Lmain\eta$.

\begin{lemma}(Informal)\label{lem:phi_12}
\begin{equation*}
\begin{split}
&\mathbb{E}_{i, j}[\phi_1(k + 1) + \phi_2(k + 1)] - \phi_1(k) - \phi_2(k)
\\ &\qquad \leq -2\pmin\begin{pmatrix}x - x^*\\ \eta (U\mathbbm{1} + m\overline{\alpha})\end{pmatrix}^T\begin{pmatrix}0 & 0 \\ 0 & I_d\end{pmatrix}\begin{pmatrix}x - x^*\\ \eta (U\mathbbm{1} + m\overline{\alpha})\end{pmatrix} - 2m\pmin\eta(x - x^*)^T\nabla f(x) \\
&\qquad\qquad + \Theta\left(\eta^2\right)\left(\frac{q}{n}\sum_i{|\Astar_i|_2^2} + \frac{q}{n}\sum_j{|u_j|_2^2} + \frac{1}{n}\sum_j\left(|g_j^*|_2^2 + |\beta_j^*|_2^2\right)\right)\\
&\qquad\qquad+ O(\eta^2)\frac{1}{n}\sum_i|\nabla f_i(x) - \nabla f_i(x^*)|_2^2.
\end{split}
\end{equation*}
\end{lemma}

The first two terms of this lemma yield a significant decrease in the potential. However, the potential may increase from the remaining second order terms, which come from the variance of the update. We can cancel the second order term involving the   $|\Astar_i|^2$, $|u_j|^2$,$|g_{j}^*|^2$, and $|\beta_{j}^*|^2$  terms by considering the expected change in potential in $\phi_3$, $\phi_4$ and $\phi_5$, captured in the next lemma, formally stated in Section~\ref{sec:main_proof}. While we use big-O notation here, as one can see in the formal lemma in Section~\ref{sec:main_proof}, the exact constants in $\phi_3$, $\phi_4$ and $\phi_5$ are chosen to cancel the second order terms in Lemma~\ref{lem:phi_12}.

\begin{lemma}\label{lem:phi_34}(Informal)
\begin{equation*}
\begin{split}
&\mathbb{E}[\phi_3(k + 1) + \phi_4(k + 1) + \phi_5(k + 1)] - (\phi_3(k) + \phi_4(k) + \phi_5(k)) \leq \left(- \frac{m\pmin}{4n}\right)\left(\phi_3(k) + \phi_4(k) + \phi_5(k)\right)
\\ &\qquad\qquad - \Theta\left(\eta^2\right)\left(\frac{q}{n}\sum_i{|\Astar_i|_2^2} + \frac{q}{n}\sum_j{|u_j|_2^2} + \frac{1}{n}\sum_j\left(|g_j^*|_2^2 + |\beta_j^*|_2^2\right)\right)
\\ &\qquad\qquad + O\left(\eta^2q\right)\frac{1}{n}\sum_i|\nabla f_i(x) - \nabla f_i(x^*)|_2^2.
\end{split}
\end{equation*}
\end{lemma}

Intuitively, the $\left(1 - \Theta\left(\frac{m\pmin}{n}\right)\right)$ contraction in this lemma is possible because in each iteration, with probality at least $\frac{m\pmin}{n}$, any one of the $n$ variables $\alpha_{i_j}$ is replaced by the variable $g_j$, which is in turn replaced by a fresh gradient $\nabla f_i(x)$.

Finally, we use the smoothness of each $f_i$ to bound $\frac{1}{n}\sum_i|\nabla f_i(x) - \nabla f_i(x^*)|_2^2$ by $\Lsup (x - x^*)^T\nabla f(x)$ (Lemma~\ref{grad_diff_norm}). This allows us to cancel all of the $\frac{1}{n}\sum|\nabla f_i(x) - \nabla f_i(x^*)|_2^2$ terms with the negative $\eta(x - x^*)\nabla f(x)$ term in Lemma~\ref{lem:phi_12}. We show that if $\eta \leq \Theta\left(\frac{1}{\Lsup + \sqrt{m\Lmain\Lsup}}\right)$, the negative $2m\pmin\eta(x - x^*)^T\nabla f(x)$ term in Lemma~\ref{lem:phi_12} dominates the positive $O\left(\eta^2q\right)\frac{1}{n}\sum_i|\nabla f_i(x) - \nabla f_i(x^*)|_2^2$ term (Claim~\ref{eq:det}). Using the $\mu$-strong convexity of $f$, we show that the remaining fraction of the negative $2m\pmin\eta(x - x^*)^T\nabla f(x)$ term leads to a negative $2m\pmin\mu|x - x^*|_2^2$ and $\frac{m^2\pmin\eta}{n}(f(x) - f(x^*))$ term.

Combining Lemma~\ref{lem:phi_12} and Lemma~\ref{lem:phi_34} with the observations above, we obtain the following lemma:

\begin{lemma}\label{lem:phi_all}
For $\eta \leq \Theta\left(\frac{1}{\Lsup + \sqrt{m\Lmain\Lsup}}\right)$,
\begin{equation*}
\begin{split}
\mathbb{E}[\phi(k + 1)] - \phi(k) &\leq -2\pmin\begin{pmatrix}y \\ \eta (U\mathbbm{1} + m\Abar)\end{pmatrix}^T\begin{pmatrix}\frac{\mu m \eta}{2} & 0 \\ 0 & 1\end{pmatrix}\begin{pmatrix}y\\ \eta(U\mathbbm{1} + m\Abar)\end{pmatrix}\\
&\qquad- \frac{m\pmin}{4n}(\phi_1(k) + \phi_3(k) + \phi_4(k) + \phi_5(k))
\end{split}
\end{equation*}
\end{lemma}

Some linear-algebraic manipulations (Lemma~\ref{minQ}) yield Proposition~\ref{single_step}.

\section{Experiments}\label{sec:experiments}
\begin{figure}[t]
\centering
\includegraphics[width=8.5cm]{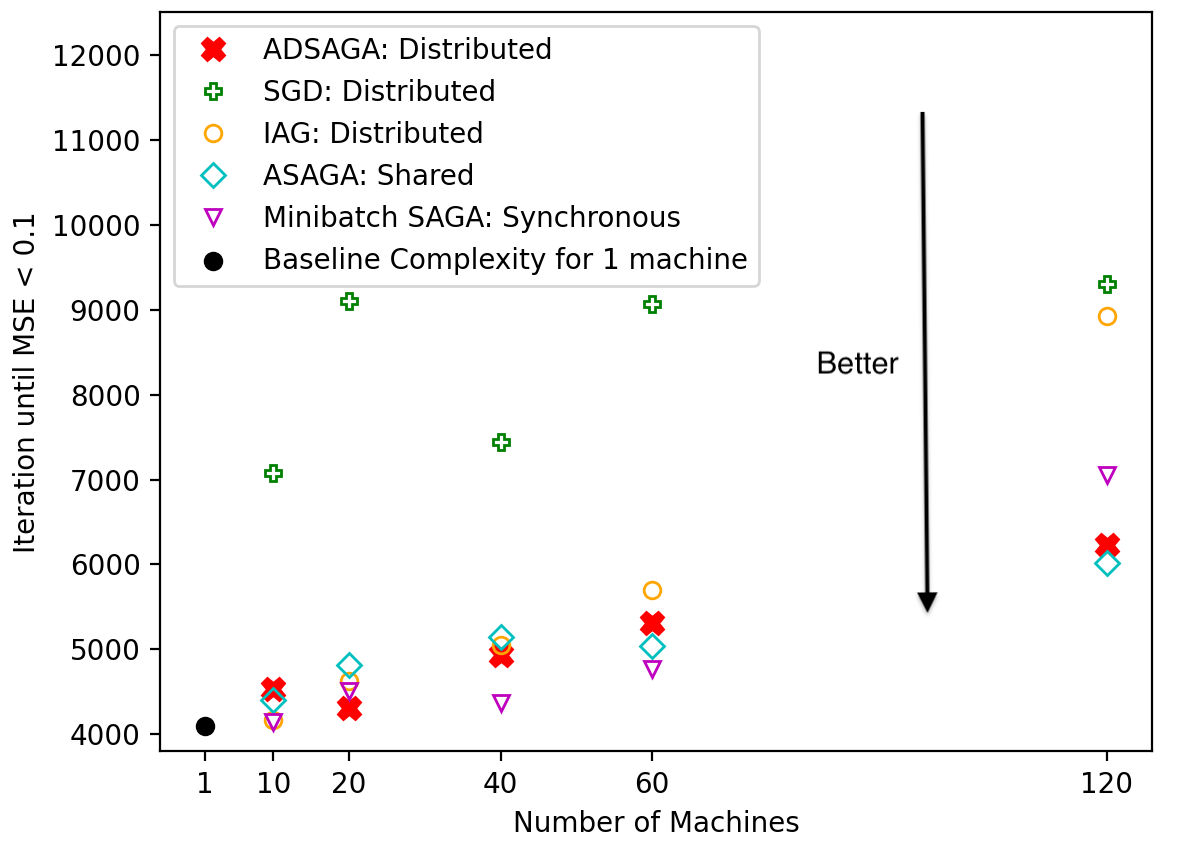}
  \caption{Comparison of ADSAGA (this work), ASAGA~\citep{asaga_journal}, minibatch SAGA~\citep{minibatch_saga}, SGD~\citep{xiangru}, and IAG~\citep{incremental}: Iteration complexity to achieve $|x_k - x^*|_2^2 \leq 0.1$, averaged over 8 runs.
  %Left: exponential work times. Right: shifted exponential work times, with a shift of 10.
  }\label{fig:its_vs_m}
\end{figure}

\begin{figure}[t]
\begin{tabular}{cc}
\includegraphics[width=8cm]{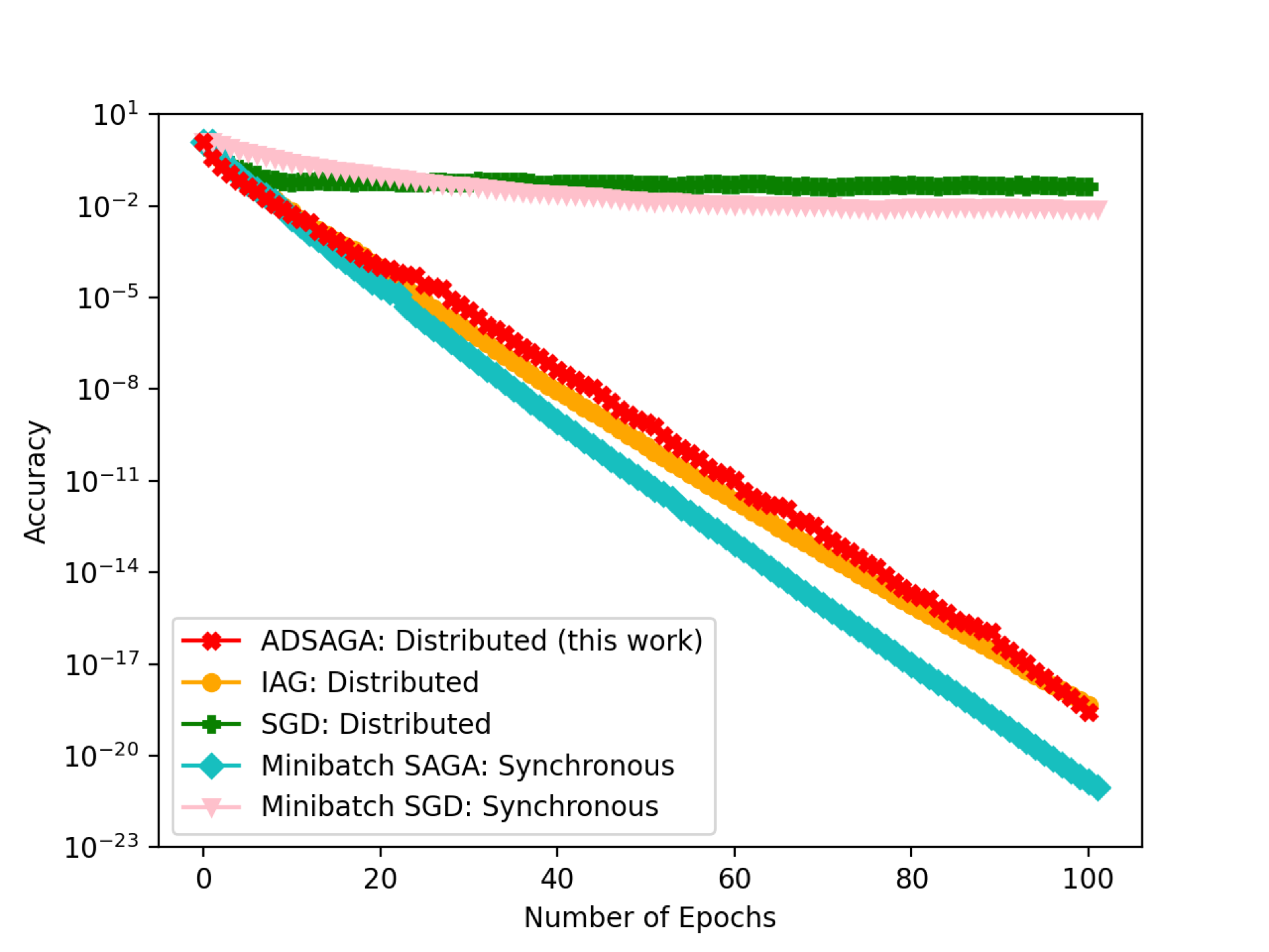} &
\includegraphics[width=8cm]{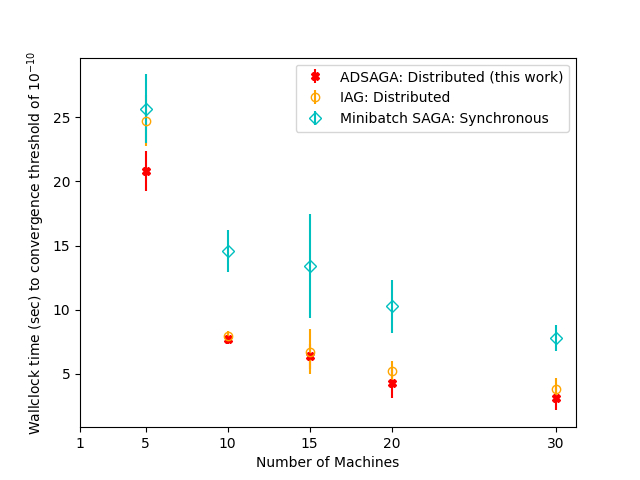} \\
(a)  & (b) \\[8pt]
\end{tabular}

\caption{(a) Convergence accuracy after 100 epochs. (b) Wallclock time to achieve $|x_k - x^*|_2^2 \leq 10^{-10}$, averaged over 8 runs.}\label{fig:wallclock}
\end{figure}

We conduct experiments to compare the convergence rates of ADSAGA to other state-of-the-art algorithms: SGD, IAG, ASAGA, and minibatch SAGA.  In our first set of experiments, we simulate the stochastic delay model of Section~\ref{sec:model}.  In our second set, we implement these algorithms in a distributed compute cluster.

\paragraph{Data.} For all experiments, we simulate these algorithms on a randomly-generated least squares problem $\min_{\hat{x}} |A\hat{x} - b|_2^2.$ Here $A \in \mathbb{R}^{n \times d}$ is chosen randomly with i.i.d. rows from $\mathcal{N}(0, \frac{1}{d}I_d)$, and $x \sim \mathcal{N}(0, I_d)$. The observations $b$ are noisy observations of the form $b = Ax + Z$, where $Z \sim \mathcal{N}(0, \sigma^2I_n)$. In the first set of experiments with simulated delays, we choose $n = 120$,  $d = 60$, and $\sigma = 1$. For the the second set of experiments on the distributed cluster, we choose a larger 10GB least squares problem with $n = 600000$ and $d = 200000$, and $\sigma = 100$.

\paragraph{Simulated Delays.}
In our first set of experiments, we  empirically validate our theoretical results by comparing the iteration complexity of these algorithms in the stochastic delay model described in Section~\ref{sec:model}, where all the $p_i$ are the same. (With one exception: when we compare to minibatch SAGA, we assume a synchronous implementation with a minibatch size $m$ equal to the number of machines).  ADSAGA, SGD, and IAG are simulated in the distributed-data setting, while ASAGA is simulated in the shared-data setting. We run ADSAGA, minibatch SAGA, ASAGA, SGD, and IAG on $\hat{x}$ with $m$ machines for $m \in 
\{10, 20, 40, 60, 120\}$. To be fair to all algorithms, for all experiments, we use a grid search to find the best step size in $\{ 0.05 \times i \}_{i \in [40]}$, ensuring that none of the best step sizes were at the boundary of this~set.

In Figure~\ref{fig:its_vs_m}, we plot the expected iteration complexity to achieve $|\hat{x} - x^*|_2^2 \leq 0.1$, where $x^* := \min_{\hat{x}} |A\hat{x} - b|_2^2$ is the empirical risk minimizer. Figure~\ref{fig:its_vs_m} demonstrates that, in the model in Section~\ref{sec:model}, ADSAGA outperforms the two alternative state-of-the-art algorithms for the distributed setting: SGD and IAG, especially as the number of machines $m$ grows.   SGD converges significantly more slowly, even for a small number of machines, due to the variance in gradient steps.

\paragraph{Distributed Experiments.}
In our second set of experiments, we run the five data-distributed algorithms (ADSAGA, SGD, IAG, minibatch-SAGA, and minibatch-SGD) in the distributed compute cluster and compare their wallclock times to convergence. We do not simulate the shared-data algorithm ASAGA because the full dataset is too large to fit in RAM, and loading the data from memory is very slow. The nodes used contained any of the following four processors: Intel E5-2640v4, Intel 5118, AMD 7502, or AMD 7742. We implement the algorithms in Python using MPI4py, an open-source MPI implementation. For ADSAGA, IAG, asynchronous SGD, and minibatch SAGA, each node stores its partition of the data in RAM. In each of the three asynchronous algorithms, at each iteration, the PS waits to receive a gradient update from any node (using MPI.ANY\_SOURCE). The PS then sends the current parameter back to that node and performs the parameter update specified by the algorithm. In the synchronous minibatch algorithm, the PS waits until updates have been received by all nodes before performing an update and sending the new parameter to all nodes. To avoid bottlenecks at the PS, the PS node checks the convergence criterion and logs progress only once per epoch in all algorithms.

We run ADSAGA, SGD, IAG, and minibatch-SAGA on $\hat{x}$ with one PS and $m$ worker nodes for $m \in 
\{5, 10, 15, 20, 30\}$. We implement the vanilla version of ADSAGA, which does not require the variables $u_j$ designed for heterogeneous update rates. In practice, while we measured substantial heterogeneity in the update rates of each machine --- with some machines making up to twice as many updates as others --- we observed that the vanilla ADSAGA implementation worked just as well as the full implementation with the $u_j$ variables. For all algorithms, we use a block size of 200 (as per Remark~\ref{remark:blocking}), and we perform a hyperparameter grid search over step sizes to find the hyperparameters which yield the smallest distance $\hat{x} - x^*$ after $100$ epochs.

In Figure~\ref{fig:wallclock}, we compare the performance of these algorithms in terms of iteration complexity and wallclock time. In Figure~\ref{fig:wallclock}(a), we plot the accuracy $|\hat{x} - x^*|_2^2$ of each algorithm after 100 epochs, where $x^* := \min_{\hat{x}} |A\hat{x} - b|_2^2$ is the empirical risk minimizer. We observe that the algorithms that do not use variance reduction (asynchronous SGD and minibatch-SGD) are not able to converge nearly as well as the variance-reduced algorithms. ADSAGA and IAG perform similarly, while minibatch-SAGA performs slightly better in terms of iteration complexity. In Figure~\ref{fig:wallclock}(b), we plot the expected wallclock time to achieve to achieve $|\hat{x} - x^*|_2^2 \leq 10^{-10}$. We only include the variance-reduced algorithms in this plot, as we were not able to get SGD to converge to this accuracy in a reasonable amount of time. Figure~\ref{fig:wallclock}(b) demonstrates that while the synchronous minibatch-SAGA algorithm may be advantageous in terms of iteration complexity alone, due to the cost of waiting for all workers to synchronize at each iteration, the asynchronous algorithms (IAG and ADSAGA) converge in less wallclock time. Both IAG and ADSAGA perform similarly. This advantage of asynchrony increases as the number of machines increases: while with 5 machines the asynchronous algorithms are only 20\% faster, with 30 machines, they are 60\% faster. Our experiments confirm that ADSAGA, the natural adaptation of SAGA to the distributed setting, is both amenable to theoretical analysis and performs well practically.

% \mg{Add plot with number of its until convergence}

% \mg{Possibly insert paragraph and plot for noisier data where SGD does worse.}

% \begin{table}[ht]
% \caption{Step Size and Block Sizes used in Figure~\ref{fig:wallclock}. Step size units are $10^{-5}$.}\label{table:hps}
% \begin{center}
% \begin{tabular}{| c | c | c | c| c| c| c| c|c|c|c| }
% \hline
% Algorithm & \multicolumn{2}{c|}{$m = 1$}& \multicolumn{2}{c|}{$m = 4$}& \multicolumn{2}{c|}{$m = 8$}& \multicolumn{2}{c|}{$m = 16$}& \multicolumn{2}{c|}{$m = 32$} \\ 
%  & \footnotesize Step & \footnotesize Block & \footnotesize Step &\footnotesize  Block & \footnotesize Step & \footnotesize  Block & \footnotesize Step &\footnotesize  Block & \footnotesize Step & \footnotesize  Block\\ 
%   \hline \hline
%   ADSAGA  & 2.5 & 16 & 3.9 & 64 & 3.6 & 512 & 3.3 & 64 & 3.3 & 256 \\
%      \hline
%       IAG & 2.1 & 512 & 2.1 & 1024 & 2.8 & 512 & 2.3 & 16 & 1.8 & 512 \\
%      \hline
%       SGD  & 6.6 & 256 & 2.8& 1024 & 2.6 & 1024 & 2.4 & 1024 & 1.7 & 1024 \\
%      \hline
%       ASAGA (shared)  & 3.3 & 1024 & 3.3 & 256& 2.3 & 256 & 2.5 & 64 & 0.5 & 1024 \\
%      \hline
%       Synchronous Minibatch  & 1.3 & 64 & 3.9 & 64 & 3.3 & 16 & 3.3 & 512 & 3.3 & 256\\
%      \hline
% \end{tabular}
% \end{center}
% \end{table}

\section{Conclusion and Open Questions}\label{sec:conclusion}
In this paper, we introduced and analyzed ADSAGA, a SAGA-like algorithm in an asynchronous, distributed-data setting. We showed that in a particular stochastic delay model, ADSAGA achieves convergence in $\tilde{O}\left(\left(n + \frac{\Lsup}{\mu} + \frac{\sqrt{m\Lmain\Lsup}}{\mu}\right)\log(1/\epsilon)\right)$ iterations. To the best of our knowledge, this is the first provable result for asynchronous algorithms in the distributed-data setting --- under any delay model --- that scales both logarithmically in $1/\epsilon$ and linearly in $n$. This work leaves open several interesting questions:

\begin{enumerate}
    \item For arbitrary but bounded delays in the distributed setting (studied in \citet{incremental, incremental_proximal, inc_prox_2}), is the dependence on $n^2$ in the iteration complexity optimal? 
    % \item Can the rate of asynchronous SAGA with delays bounded by $O(m)$ (either with shared or distributed data) be improved to $\tilde{O}\left(n + \frac{\Lsup}{\mu} + \frac{m\Lmain}{\mu}\right)$, such that the dependence on $m$ scales with $\Lmain$ instead of $\Lsup$? This would show that asynchronous SAGA performs as well as its parallel, synchronous counterpart, minibatch SAGA. Some evidence towards this is given in \citet{fixed_delay, consistent} for quadratic objectives. Our empirical results also suggest that ASAGA performs as well as minibatch SAGA.
    \item In the \textit{decentralized} random gossip setting of Figure~\ref{fig:settings}(c), what rates does a SAGA-like algorithm achieve?
    \item How would the local-SGD or local minibatching approaches, which are popular in federated learning, perform in the presence of stochastic delays?
\end{enumerate}

%%% TODO: put these back in if accepted!

%\section*{Acknowledgements}
%The authors thank Ioannis Mitliagkas, Kevin Tian, and Aaron Sidford for helpful conversations and pointers to the literature. Ilai, Masha

\bibliographystyle{plainnat}
\bibliography{refs}

\appendix

\section{Proof of ADSAGA Convergence}\label{sec:main_proof}

In this section we prove Theorem~\ref{thm:main}, restated as Theorem A.1. 
\mainthm*

\begin{remark}
Due to the strong convexity, we have $|x - x^*|_2^2 \leq \frac{\mu}{2}(f(x) - f(x^*))$, so this theorem also implies that $|x^k - x^*|_2^2 \leq \epsilon$ after $\rate$ iterations.
\end{remark}

We begin by establishing notation and reviewing the update performed at each step of Algorithm~\ref{local_stoc_data}.

Recall that $x$ is the value held at the PS, and let $y$ denote $x - x^*$. Let $G$ and $H$ be the matrices whose $j$th column contains the vector $g_j$ and $h_j$ respectively. For all the variables in Algorithm~\ref{local_stoc_data} and discussed above, we use a superscript $k$ to denote their value at the beginning of iteration $k$. When the iteration $k$ is clear from context, we will eliminate the superscripts $k$. To further simplify, we will use the following definitions: $\Astar_i := \alpha_i - \nabla f_i(x^*)$, $\beta^*_j := \beta_j - \nabla f_{i_j}(x^*)$, and $g^*_j := g_j - \nabla f_{i_j}(x^*)$, where $i_j$ is the index of the function used by machine $j$ to compute $g_j$ and $\beta_j$, as indicated in Algorithm~\ref{local_stoc_data}.

Recall that we analyze the expectation of the following potential function $\phi(x, G, H, U, \alpha, \beta)$:

$$\phi(x, G, H, U, \alpha, \beta) := \sum_{\ell = 1}^5{\phi_{\ell}(x, G, H, U, \alpha, \beta)},$$ where

\begin{equation*}
\begin{split}
\phi_1(x, G, H, U, \alpha, \beta) &:= c_1\left(f(x) - f(x^*)\right), \\
\phi_2(x, G, H, U, \alpha, \beta) &:= \begin{pmatrix}x - x^*\\ \eta (U\mathbbm{1} + m\overline{\alpha})\end{pmatrix}^T\begin{pmatrix}1 & -1 \\ -1 & 2\end{pmatrix}\begin{pmatrix}x - x^* \\ \eta(U\mathbbm{1} + m\overline{\alpha})\end{pmatrix}, \\
\phi_3(x, G, H, U, \alpha, \beta) &:= \eta^2c_3\sum_j{\frac{\pmin}{p_j}|g_j^*|_2^2}, \\
\phi_4(x, G, H, U, \alpha, \beta) &:= \eta^2c_4\left(2\sum_i{\frac{\pmin}{p_j}|\alpha_i^*|_2^2} - \sum_j{\frac{\pmin}{p_j}|\beta_j^*|_2^2}\right), \\
\phi_5(x, G, H, U, \alpha, \beta) &:= \eta^2c_5\sum_j{|u_j|_2^2}.
\end{split}
\end{equation*}

Above, we abbreviate the $2d \times 2d$ matrix $\begin{pmatrix}I_d & -I_d \\ -I_d & 2I_d\end{pmatrix}$ as $\begin{pmatrix}1 & -1 \\ -1 & 2\end{pmatrix}$. Later, we will choose the $c_{\ell}$ as follows:
\begin{equation*}
\begin{split}
c_1 &= 4m\eta,\\
c_3 &=  \left(64 + 168\frac{m}{n}\left(\frac{\pmax}{\pmin}\right)^2\right)c_5\\
c_4 &=  \left(22 + \frac{76m}{n}\left(\frac{\pmax}{\pmin}\right)^2\right)c_5,\\
c_5 &= \frac{4}{3}\left(4m\Lmain\eta + 4\right).
\end{split}
\end{equation*}
The following proposition is our main technical proposition. 
\mainprop*

Before we prove this proposition, we prove Theorem~\ref{thm:main}, which follows from Proposition~\ref{single_step}.

\begin{proof}(Theorem~\ref{thm:main})
Upon initialization, the expected potential $\mathbb{E}[\phi(0)]$ (over the random choices of $i_j$) equals
\begin{equation*}
\begin{split}
\mathbb{E}[\phi(0)] &= 
4m\eta(f(x^0) - f(x^*)) + |x^0 - x^*|_2^2 + \eta^2\left(m(c_3 - c_4) + 2nc_4\right)\frac{1}{n}\sum_i|\nabla f_i(x^*)|_2^2 \\
&\leq \left(4m\eta + \frac{2}{\mu}\right)\left(f(x^0) - f(x^*)\right) + \eta^2
\left(c_3 + c_4\right)\sum_i|\nabla f_i(x^*)|_2^2\\
&\leq \left(4m\eta + \frac{2}{\mu}\right)\left(f(x^0) - f(x^*)\right) + \eta^2n(c_3 + c_4)\sigma^2,
\end{split}
\end{equation*}
where the second line uses $\mu$-strong convexity, the fact that $c_3 > c_4$ and that $m \leq n$.
For $\eta = \etabd$, we use the fact that $\eta \leq \frac{1}{2r\Lsup}$ since $\frac{c_3 + c_4}{2r} \leq m\Lmain\eta + 1 \leq m$; thus
\begin{equation*}
\begin{split}
\mathbb{E}[\phi(0)] &\leq \left(4m\eta + \frac{2}{\mu}\right)\left(f(x^0) - f(x^*)\right) + \left(\frac{\eta n(c_3 + c_4)}{2r\Lsup}\right)\sigma^2\\
&\leq \left(4m\eta + \frac{2}{\mu}\right)\left(f(x^0) - f(x^*)\right) + \left(\frac{2\eta mn}{\Lsup}\right)\sigma^2,
\end{split}
\end{equation*}
where we have plugged in the definition of $r$.

With $\gamma = m\pmin \min\left(\frac{1}{4n}, \eta \mu\right)$ as in Proposition~\ref{single_step} and $\eta$ and $k$ as in Theorem~\ref{thm:main}, we have 
$$k\gamma \geq \log\left(\frac{\left(1 + \frac{1}{2m\mu\eta}\right)\left(f(x^0) - f(x^*)\right) + \frac{n\sigma^2}{2\Lsup}}{\epsilon}\right),$$ and so
\begin{equation*}
\begin{split}
\left(1 - \gamma\right)^k &\leq \exp\left(-\gamma k\right) \leq \exp\left(-\log\left(\frac{\left(1 + \frac{1}{2m\mu\eta}\right)\left(f(x^0) - f(x^*)\right) + \frac{n\sigma^2}{2\Lsup}}{\epsilon}\right)\right) \\
&\leq \frac{4m\eta\epsilon}{\left(4m\eta + \frac{2}{\mu}\right)\left(f(x^0) - f(x^*)\right) + \frac{2\eta mn\sigma^2}{\Lsup}} \leq \frac{4m\eta\epsilon}{\mathbb{E}[\phi(0)]}.
\end{split}
\end{equation*}

It follows from Proposition~\ref{single_step} that $$\mathbb{E}[\phi(k)] \leq 4m\eta\epsilon$$

Since $\phi(k) \geq 4m\eta(f(x^k) - f(x^*))$, we have $\mathbb{E}[f(x^k) - f(x^*)] \leq \epsilon$ as desired. 
\end{proof}

\begin{proof}(Proposition~\ref{single_step})
To abbreviate, let $M = \begin{pmatrix}1 & -1 \\ -1 & 2\end{pmatrix}$. We also abbreviate $\nabla_i := \nabla f_i(x^k)$, and let $\nabla$ be the matrix whose $i$th column is $\nabla_i$.
% \mkw{Since $i$ is the special index that we pick in the algorithm, maybe don't use it as a dummy variable in the above?}\mg{I thought it was more clear this way?}
All expectations are over the random choice of $j \sim \mathcal{P}$ and $i \sim \text{Uniform}(S_j)$ in Algorithm~\ref{local_stoc_data}. Because, each function $i$ only belongs to a single machine $j(i)$, we will sometime abbreviate the choice of $i, j$ in each iteration as just a choice of $i$; when we do so, any otherwise unspecified use of $j$ should be interpreted as the machine $j(i)$ where $f_i$ is stored.

% \begin{proof}(Lemma~\ref{unbiased})
Recall that in each iteration, given the random choice of $j$ and $i$ in that iteration, the following updates to the variables in Algorithm~\ref{local_stoc_data} are made, with $\eta_j = \eta \frac{\pmin}{p_j}$:
\begin{equation}\tag{$\star$}
\begin{split}
x^{k + 1} &\leftarrow x^k - \eta_j(u_j^k + \Abar^k);\\
\Abar^{k + 1} & \leftarrow \Abar^k + \frac{1}{n}h_{j}^k;\\
\alpha_{i_{j}}^{k + 1} & \leftarrow g_{j}^{k};\\
g_{j}^{k + 1} & \leftarrow \nabla f_i(x^k);\\
\beta_j^{k + 1} & \leftarrow \begin{cases}
\alpha_i^k & i \neq i_{j}; \\
g_{j}^k & i = i_{j};
\end{cases} \\
h_{j}^{k + 1} & \leftarrow g_j^{k+1} - \beta_j^{k+1} = \begin{cases}
\nabla f_i(x^k) - \alpha_i^k & i \neq i_{j}; \\
\nabla f_i(x^k) - g_{j}^k & i = i_{j}.
\end{cases}\\
u^{k + 1}_{j} &\leftarrow  u_j^k\left(1 - \frac{\pmin}{p_j}\right) + \frac{\pmin}{p_j}h_{j}^{k+1} - \frac{m}{n}\left(1 - \frac{\pmin}{p_j}\right)h_{j}^k \\
i_j^{k+1} & \leftarrow i.
\end{split}
\end{equation}
Note that the update to $u_j$ contains a reference to both $h^{k + 1}_j$ and $h^k_j$. At the end of each iteration, we maintain the invariant $\alpha_{i_j}^{k + 1} = \beta_j^{k + 1}$, and $\Abar^{k + 1} = \frac{1}{n}\sum_i \alpha_i^{k + 1}$ because 
\begin{equation*}h_{j}^k = g_{j}^k - \beta_{j}^k = \alpha_{i_{j}}^{k + 1} - \alpha_{i_{j}}^k.
\end{equation*}

Dropping the superscripts of $k$, let $$\Delta_i := \begin{pmatrix}x^{k + 1} - x \\ \eta(U^{k+1}\mathbbm{1} + m\Abar^{k+1}) - \eta(U\mathbbm{1} + m\Abar)\end{pmatrix}$$ be the change to the vector  $\begin{pmatrix}x \\ U\mathbbm{1} + m\Abar\end{pmatrix}$ if function $i$ is chosen at iteration $k$. We begin by computing $\Delta_i$.

\begin{claim}\label{delta_i}
\begin{equation*}
\Delta_{i} = \eta_j\begin{pmatrix}-u_j - \Abar \\ -u_j + \nabla_i - \alpha_i +  \left(-I(i = i_{j} + \frac{m}{n})\right)h_j\end{pmatrix}.
\end{equation*}
\end{claim}

\begin{proof}
If $i \neq i_{j}$, then we have 
\begin{equation*}
    U^{k + 1}\mathbbm{1} = U\mathbbm{1} + \frac{\pmin}{p_j}\left(-u_{j} + \nabla f_i(x) - \alpha_i\right) - \frac{m}{n}\left(1 - \frac{\pmin}{p_j}\right)h_j.
\end{equation*}
Otherwise if $i = i_{j}$, then
% \begin{equation*}
% \begin{split}
%     U^{k + 1}\mathbbm{1} &= U\mathbbm{1} - \frac{\pmin}{p_j}u_{j} + \nabla f_i(x) - g_{j} \\
%     &= U\mathbbm{1} - \frac{\pmin}{p_j}u_{j} + \nabla f_i(x) - \alpha_i - g_{j} + \beta_{j}\\
%     &= U\mathbbm{1} - \frac{\pmin}{p_j}u_{j} + \nabla f_i(x) - \alpha_i - h_{j}
% \end{split}
% \end{equation*}
\begin{equation*}
\begin{split}
    U^{k + 1}\mathbbm{1} &= U\mathbbm{1} + \frac{\pmin}{p_j}\left(-u_{j} + \nabla f_i(x) - g_{j}\right)  - \frac{m}{n}\left(1 - \frac{\pmin}{p_j}\right)h_j\\
    &= U\mathbbm{1} + \frac{\pmin}{p_j}\left(-u_{j} + \nabla f_i(x) - \alpha_i - g_{j} + \beta_{j}\right)  - \frac{m}{n}\left(1 - \frac{\pmin}{p_j}\right)h_j\\
    &= U\mathbbm{1} + \frac{\pmin}{p_j}\left(-u_{j} + \nabla f_i(x) - \alpha_i - h_{j}\right)  - \frac{m}{n}\left(1 - \frac{\pmin}{p_j}\right)h_j
\end{split}
\end{equation*}

where the second line follows because $\alpha_i = \beta_{j}$ in this case.
In both cases, we have 
\begin{equation*}
    m\Abar^{k + 1} = m\Abar + \frac{m}{n}h_{j}.
\end{equation*}

Putting this all together with the update to $x$ and the fact that $\eta_j = \eta\frac{\pmin}{p_j}$ yields the claim:

\begin{equation*}
\begin{split}
\Delta_{i} &= \begin{pmatrix}-\eta_j(u_j + \Abar) \\ \eta\left(\frac{\pmin}{p_j}\left(-u_j + \nabla_i - \alpha_i\right) +  \left(-\frac{m}{n}\left(1 -\frac{\pmin}{p_j}\right) - \frac{\pmin}{p_j}I(i = i_{j}) + \frac{m}{n}\right)h_j\right)\end{pmatrix} \\
&= \begin{pmatrix}-\eta_j(u_j + \Abar) \\ \eta\frac{\pmin}{p_j}\left(-u_j + \nabla_i - \alpha_i +  \left(- I(i = i_{j})+\frac{m}{n}\right)h_j\right)\end{pmatrix}\\
&= \begin{pmatrix}-\eta_j(u_j + \Abar) \\ \eta_j\left(-u_j + \nabla_i - \alpha_i +  \left( - I(i = i_{j}) + \frac{m}{n}\right)h_j\right)\end{pmatrix}
\end{split}
\end{equation*}
\end{proof}

Computing the expectation of $\Delta_i$ and plugging in $\eta_j = \eta\frac{\pmin}{p_j}$ yields the Unbiased Trajectory lemma:
\begin{customlem}{\ref{unbiased}}[Unbiased Trajectory]
\begin{equation*}
\mathbb{E}_{i, j}\left[\Delta_i\right] = -\pmin \eta\begin{pmatrix}U\mathbbm{1} + m\overline{\alpha} \\ U\mathbbm{1} + m\overline{\alpha} - \frac{m}{n}\nabla\mathbbm{1}\end{pmatrix}.
\end{equation*}
\end{customlem}
\begin{proof}
First note that $I(i = i_j)$ is independent from $j$, and occurs with probability $\frac{1}{|S_j|} = \frac{m}{n}$. Hence, from Claim~\ref{delta_i},
we have 

\begin{equation*}
\begin{split}
\mathbb{E}_{i, j}\left[\Delta_i\right] &= \mathbb{E}_{i, j}\eta_j\begin{pmatrix}-u_j - \Abar \\ -u_j + \nabla_i - \alpha_i +  \left(-I(i = i_{j} + \frac{m}{n})\right)h_j\end{pmatrix} \\
&= \sum_{j, i \in S_j}\frac{m}{n}p_j\eta_j\begin{pmatrix}-u_j - \Abar \\ -u_j + \nabla_i - \alpha_i +  \left(-I(i = i_{j} + \frac{m}{n})\right)h_j\end{pmatrix} \\
&= \pmin \eta\begin{pmatrix}-U\mathbbm{1} - m\Abar \\ -U\mathbbm{1} + \frac{m}{n}\nabla\mathbbm{1} - m\Abar\end{pmatrix}.
\end{split}
\end{equation*}
\end{proof}

% The first statement is true because the uniformly random choice of $i$ from $[n]$ induces a uniformly random choice of $j$ from $[m]$.

% so 
% \begin{equation}\label{eq:abar}
%     \mathbb{E}_{i, j}[\Abar^{k + 1}] = \Abar^k + \frac{1}{mn}H^k\mathbbm{1}.
% \end{equation}

% Now $i = i_{j}$ with probability $\frac{m}{n}$, and so 
% \begin{equation}
% \begin{split}
%     \mathbb{E}_{i, j}[H^{k + 1}\mathbbm{1}] &= \left(1 - \frac{1}{m} - \frac{1}{n}\right)H^k\mathbbm{1} + \nabla f(x^k) - \mathbb{E}_{i, j}[\alpha^k_i] \\
%     & = \left(1 - \frac{1}{m} - \frac{1}{n}\right)H^k\mathbbm{1} + \nabla f(x^k) - \Abar^k,
% \end{split}
% \end{equation}
% where the second line follows from the invariant $\Abar = \frac{1}{n}\sum_i[\alpha_i]$.

% Combining this with Equation~\ref{eq:abar}, we have 

% \begin{equation}
% \begin{split}
%     \mathbb{E}_{i, j}[H^{k + 1}\mathbbm{1} + m\Abar^{k + 1}] &= \left(1 - \frac{1}{m} - \frac{1}{n}\right)H^k\mathbbm{1} + \nabla f(x^k) - \mathbb{E}_{i, j}[\alpha^k_i] + m\Abar^k + \frac{1}{n}H^k\mathbbm{1} \\
%     &= \left(1 - \frac{1}{m}\right)H^k\mathbbm{1} + \nabla f(x^k) - \mathbb{E}_{i, j}[\alpha^k_i] + m\Abar^k\\
%     &= \left(1 - \frac{1}{m}\right)H^k\mathbbm{1} + \nabla f(x^k) + (m - 1)\Abar^k \\
%     &=  \left(1 - \frac{1}{m}\right)\left(H^k\mathbbm{1} + m\Abar^k\right) + \nabla f(x^k).
% \end{split}
% \end{equation}
% Dividing by $m$ yields the lemma.
% \end{proof}

Now
\begin{equation*}\phi_2(k + 1) = \left(\begin{pmatrix}y \\ \eta(U\mathbbm{1} + m\overline{\alpha})\end{pmatrix} + \Delta_{i}\right)^T\begin{pmatrix}1 & -1 \\  -1 & 2\end{pmatrix}\left(\begin{pmatrix}y \\ \eta(U\mathbbm{1} + m\overline{\alpha})\end{pmatrix} + \Delta_{i}\right),\end{equation*}

so we can compute the difference
\begin{equation}\label{eq:phi_2}
\begin{split}
\mathbb{E}[\phi_2(k + 1)] - \phi_2(k) &= 2\begin{pmatrix}y \\ \eta (U\mathbbm{1} + m\overline{\alpha})\end{pmatrix}^T\begin{pmatrix}1 & -1 \\  -1 & 2\end{pmatrix}\mathbb{E}_{i, j}[\Delta_{i}] + \mathbb{E}_{i, j}\left[ \Delta_{i}^T M \Delta_{i}\right] \\
&= -2\eta \pmin\begin{pmatrix}y \\ \eta (U\mathbbm{1} + m\overline{\alpha})\end{pmatrix}^T\begin{pmatrix}1 & -1 \\  -1 & 2\end{pmatrix}\begin{pmatrix} U\mathbbm{1} + m\overline{\alpha} \\ U\mathbbm{1} + m\overline{\alpha} - \frac{m}{n}\nabla\mathbbm{1} \end{pmatrix} + \mathbb{E}_{i, j}\left[ \Delta_{i}^T M\Delta_{i}\right] \\
&= -2\pmin\begin{pmatrix}y \\ \eta(U\mathbbm{1} + m\overline{\alpha})\end{pmatrix}^T\begin{pmatrix}0 & 0 \\ 0 & 1\end{pmatrix}\begin{pmatrix}y\\ \eta(U\mathbbm{1} + m\overline{\alpha})\end{pmatrix} + \mathbb{E}_{i, j}\left[ \Delta_{i}^T M\Delta_{i}\right]\\  &\qquad \qquad - \frac{2\eta m \pmin}{n}y^T\nabla\mathbbm{1} + \frac{4\eta^2m\pmin}{n}\left(U\mathbbm{1} + m\overline{\alpha}\right)^T\nabla\mathbbm{1}.
\end{split}
\end{equation}

We bound the quadratic term in the difference $\mathbb{E}[\phi_2(k)] - \phi_2(k)$ in \eqref{eq:phi_2} in the following claim.
\begin{claim}\label{quad_term}
\begin{equation*}
\mathbb{E}_{i, j}\left[ \Delta_{i}^T M \Delta_{i}\right] \leq 
\frac{4m\pmin\eta^2}{n}\left(4\sum_j|g_{j}^*|_2^2+ 4\sum_j|\beta_j^*|_2^2 + \frac{n}{m}\sum_j{|u_j|_2^2} + 2\sum_i{|\Astar_i|_2^2 + 2\sum_i|\nabla_i - \nabla_i(x^*)|_2^2}\right).
\end{equation*}
\end{claim}

\begin{proof}
\begin{equation}\label{eq:quad_exp}
\begin{split}
\frac{1}{\eta^2}\mathbb{E}_{i, j}&\left[ \Delta_{i}^T M \Delta_{i}\right]\\
&= \mathbb{E}_{i, j}\left[\frac{\eta_j^2}{\eta^2}\begin{pmatrix}u_j + \Abar \\ u_j - \nabla_i + \alpha_i +  \left(-\frac{m}{n} + I(i = i_{j})\right)h_j\end{pmatrix}^T M \begin{pmatrix}u_j + \Abar \\ u_j - \nabla_i + \alpha_i +  \left(-\frac{m}{n} + I(i = i_{j})\right)h_j\end{pmatrix}\right] \\
&\leq \frac{m \pmin}{n}\sum_i{\begin{pmatrix}u_j + \Abar \\ u_j - \nabla_i + \alpha_i +  \left(-\frac{m}{n} + I(i = i_{j})\right)h_j\end{pmatrix}^T M \begin{pmatrix}u_j + \Abar \\ u_j - \nabla_i + \alpha_i +  \left(-\frac{m}{n} + I(i = i_{j})\right)h_j\end{pmatrix}}\\
&\leq \frac{4m\pmin}{n}\left(1 - \frac{m}{n}\right)\sum_i\left[\begin{pmatrix}0 \\ -\frac{m}{n}h_j \end{pmatrix}^TM\begin{pmatrix} 0 \\ - \frac{m}{n}h_j\end{pmatrix}\right] + \frac{4m\pmin}{n}\frac{m}{n}\sum_i\left[\begin{pmatrix}0 \\ \left(1 - \frac{m}{n}\right)h_j \end{pmatrix}^TM\begin{pmatrix} 0 \\\left(1 - \frac{m}{n}\right)h_j\end{pmatrix}\right] \\
&\qquad+\frac{4m\pmin}{n}\sum_i\left[\begin{pmatrix}u_j\\ u_j \end{pmatrix}^TM\begin{pmatrix}u_j\\ u_j \end{pmatrix}\right]\\
&\qquad+\frac{4m\pmin}{n}\sum_i\left[\begin{pmatrix}\overline{\alpha} \\ 
 \alpha_i - \nabla_i(x^*) \end{pmatrix}^TM\begin{pmatrix} \overline{\alpha} \\ \alpha_i - \nabla_i(x^*) \end{pmatrix}\right]
\\
&\qquad+\frac{4m\pmin}{n}\sum_i\left[\begin{pmatrix}0 \\ 
\nabla_i - \nabla_i(x^*) \end{pmatrix}^TM\begin{pmatrix} 0 \\ \nabla_i - \nabla_i(x^*) \end{pmatrix}\right].
\end{split}.
\end{equation}
Here the first inequality is by the fact the definition $\eta_j = \frac{\eta \pmin}{p_j} \leq \eta$, and the second is by Jensen's inequality and the fact that the distribution of $j$ conditioned on the indicator $I(i = i_{j})$ is equivalent to the distribution of $j$ when $i$ is chosen uniformly at random. 

We can bound each of these terms by plugging in $M = \begin{pmatrix}1 & - 1\\ - 1 & 2\end{pmatrix}$.

For the first two terms involving $h_j$, we have,
\begin{equation}
\begin{split}
\left(1 - \frac{m}{n}\right)&\sum_i\left[\begin{pmatrix}0 \\ -\frac{m}{n}h_j \end{pmatrix}^TM\begin{pmatrix} 0 \\ - \frac{m}{n}h_j\end{pmatrix}\right] + \frac{m}{n}\sum_i\left[\begin{pmatrix}0 \\ \left(1 - \frac{m}{n}\right)h_j \end{pmatrix}^TM\begin{pmatrix} 0 \\\left(1 - \frac{m}{n}\right)h_j\end{pmatrix}\right] \\
&= 2\left(\frac{m}{n} - \frac{m^2}{n^2}\right) \sum_i|h_j|_2^2\\
&= 2\left(1 - \frac{m}{n}\right)\sum_j|h_j|_2^2\\
&\leq 4\sum_j|g_j^*|_2^2 + 4\sum_j|\beta_j^*|_2^2.
\end{split}
\end{equation}
Similarly,
\begin{equation}
\sum_i\left[\begin{pmatrix}u_j\\ u_j \end{pmatrix}^TM\begin{pmatrix}u_j\\ u_j \end{pmatrix}\right] =  \sum_i|u_j|_2^2 =  \frac{n}{m}\sum_j|u_j|_2^2.
\end{equation}

For the final two terms, we have 
\begin{equation}
\begin{split}
&\sum_i\left[\begin{pmatrix}\overline{\alpha} \\ 
 \alpha_i - \nabla_i(x^*) \end{pmatrix}^TM\begin{pmatrix} \overline{\alpha} \\ \alpha_i - \nabla_i(x^*) \end{pmatrix}\right] + \sum_i\left[\begin{pmatrix}0 \\ 
\nabla_i - \nabla_i(x^*) \end{pmatrix}^TM\begin{pmatrix} 0 \\ \nabla_i - \nabla_i(x^*) \end{pmatrix}\right]\\
&\leq \sum_i\left[\begin{pmatrix}\overline{\alpha} \\ 
 \alpha_i - \nabla_i(x^*) \end{pmatrix}^TM\begin{pmatrix} \overline{\alpha} \\ \alpha_i - \nabla_i(x^*) \end{pmatrix}\right] + \sum_i\left[\begin{pmatrix}0 \\ 
\nabla_i - \nabla_i(x^*) \end{pmatrix}^TM\begin{pmatrix} 0 \\ \nabla_i - \nabla_i(x^*) \end{pmatrix}\right]\\
& =  2\sum_i|\Astar_i|_2^2 - \sum_i|\Abar|_2^2 + 2\sum_i|\nabla_i - \nabla_i(x^*)|_2^2 \\
& \leq  \frac{n}{m}\sum_j|u_j|_2^2 + 2\sum_i|\Astar_i|_2^2 + 2\sum_i|\nabla_i - \nabla_i(x^*)|_2^2.
\end{split}
\end{equation}
Plugging these three equations into Equation~\ref{eq:quad_exp} yields the claim.

\end{proof}

Next we bound the expected change in $\phi_1$.
\begin{claim}\label{claim:phi_1}
\begin{equation*}
\begin{split}
\mathbb{E}[\phi_1(k + 1)] - \phi_1(k) &\leq -\frac{c_1\eta \pmin}{n}\left(U\mathbbm{1} + m\Abar \right)^T\nabla\mathbbm{1}\\
&\qquad + c_1\Lmain\eta^2 m\pmin\left(\frac{1}{m}\sum_j|u_j|_2^2  + \frac{1}{n}\sum_i|\Astar_i|_2^2\right)
\end{split}
\end{equation*}
\end{claim}

\begin{proof}
Using convexity and $\Lmain$-smoothness, 
\begin{equation*}
\begin{split}
\mathbb{E}[\phi_1(k + 1)] - \phi(k) &=
c_1\mathbb{E}_{i, j}[f(x - \eta_j(u_j + \overline{\alpha})] - c_1f(x)\\
& \leq c_1\mathbb{E}[-\eta_j(u_j + \Abar)^T]\frac{\nabla\mathbbm{1}}{n} + \frac{c_1\Lmain}{2}\mathbb{E}[|\eta_j(u_j + \Abar)|_2^2] \\
&= -\frac{c_1\eta \pmin}{n}\left(U\mathbbm{1} + m\Abar\right)^T\nabla\mathbbm{1} + \frac{c_1\Lmain}{2}\sum_j{p_j\eta_j^2|u_j + \overline{\alpha}|_2^2}\\
& \leq -\frac{c_1\eta \pmin}{n}\left(U\mathbbm{1} + m\Abar\right)^T\nabla\mathbbm{1} + c_1\Lmain\sum_j{p_j\eta_j^2|u_j|_2^2} + c_1\Lmain\eta^2 m\pmin|\Abar|^2 \\
& \leq -\frac{c_1\eta \pmin}{n}\left(U\mathbbm{1} + m\Abar\right)^T\nabla\mathbbm{1} + c_1\Lmain \eta^2 \pmin \sum_j{|u_j|_2^2} + \frac{c_1\Lmain\eta^2 m \pmin}{n}\sum_i|\Astar_i|^2.
\end{split}
\end{equation*}
Here we used Jensen's in the second inequality, and the fact that $\Abar = \frac{1}{n}\sum_i[\Astar_i + \nabla_i(x^*)]$, so because $\frac{1}{n}\sum_i{\nabla_i(x^*)} = 0$, we have $|\Abar|^2 \leq \frac{1}{n}\sum_i|\Astar_i|_2^2$ in the third inequality.
\end{proof}

We now combine \eqref{eq:phi_2}, Claim~\ref{quad_term}, and Claim~\ref{claim:phi_1}, plugging in our choice of $c_1 = 4m\eta$, which was chosen so that the $(U\mathbbm{1} + m\Abar)^T\nabla\mathbbm{1}$ terms cancel. This yields the following lemma.

\begin{customlem}{\ref{lem:phi_12}}(Formal)
\begin{equation*}
\begin{split}
&\mathbb{E}[\phi_1(k + 1) + \phi_2(k + 1)] - (\phi_1(k) + \phi_2(k)) \leq \\ 
& \qquad -2\pmin\begin{pmatrix}y \\ \eta (U\mathbbm{1} + m\Abar)\end{pmatrix}^T\begin{pmatrix}0 & 0 \\ 0 & 1\end{pmatrix}\begin{pmatrix}y\\ \eta(U\mathbbm{1} + m\Abar)\end{pmatrix} - \left(\frac{2\eta m\pmin}{n}\right)\left(y^T\nabla\mathbbm{1}\right)\\
&\qquad\qquad+ \sum_j|u_j|_2^2\eta^2\pmin\left(4m\Lmain\eta + 4\right) \\
&\qquad\qquad+ \sum_j|g^*_j|_2^2\left(\frac{16m\pmin\eta^2}{n}\right)\\
&\qquad\qquad+ \sum_j|\beta^*_j|_2^2\left(\frac{16m\pmin\eta^2}{n}\right) \\
&\qquad\qquad+ \sum_i|\Astar_i|_2^2\frac{\eta^2 m\pmin}{n}\left(4m\eta\Lmain + 16\right) \\
&\qquad\qquad + \sum_i\gradi\left(\frac{8m\pmin\eta^2}{n}\right).
\end{split}
\end{equation*}
\end{customlem}

We proceed in the following claims to bound the differences in expectation of $\phi_3$ and $\phi_4$.

\begin{claim}\label{claim:phi_3}
\begin{equation*}
\mathbb{E}[\phi_3(k + 1)] - \phi_3(k) \leq -\eta^2c_3\pmin \sum_j\left[|g^*_j|_2^2\right] + \eta^2c_3\pmin \frac{m}{n}\sum_i|\nabla_i - \nabla_i(x^*)|_2^2.\end{equation*}
\end{claim}

\begin{proof}
If function $i$ is chosen in iteration $k$, then $g_{j}$ becomes $\nabla_i$, so
\begin{equation*}
\phi_3(k + 1) = \eta^2c_3\left(\sum_{j' \neq j}{\frac{\pmin}{p_{j'}}|g_{j'}^*|_2^2} + \frac{\pmin}{p_j}|\nabla_i - \nabla_i(x^*)|_2^2\right).
\end{equation*}
Taking the expectation over $i$ yields the claim.
% \begin{equation*}
% \begin{split}
% \phi_3(k) - \mathbb{E}[\phi_3(k + 1)] &= \eta^2c_3\left(\mathbb{E}_{i, j}\left[|g_j^*|_2^2\right] - \mathbb{E}_{i, j}\left[|\nabla_i - \nabla_i(x^*)|_2^2\right]\right).
% \end{split}
% \end{equation*}
% \mkw{Above, it might be nice to have some prose explaining that in the next step, $g_j \gets \nabla_i$ for a random $i$, and $i_j \gets i$, hence the new form of $\phi_3(k+1)$ looks like [the term on the right].  (Actually, some prose like this would have been helpful for all of these lemmas).}
\end{proof}

\begin{claim}\label{claim:phi_4}
\begin{equation*}
\begin{split}
\mathbb{E}_{i, j}[\phi_4(k + 1)] - \phi_4(k) &\leq -\frac{\pmin m}{4n}\phi_4(k) - \eta^2\frac{m \pmin c_4}{2n}\sum_i|\Astar_i|_2^2 \\
&\qquad -\eta^2\frac{\pmin c_4}{4}\sum_j|\beta^*_j|_2^2\\ 
&\qquad+ \eta^2 2\pmin c_4\sum_j|g^*_{j}|_2^2.
\end{split}
\end{equation*}
\end{claim}
\begin{proof}
If function $i$ is chosen in the $k$th iteration, then 
\begin{equation*}
\begin{split}
\phi_4(k + 1) - \phi_4(k) &= \eta^2 2c_4\frac{\pmin}{p_j}\left(|g^*_{j}|_2^2 - |\Astar_{i_{j}}|_2^2\right) \\
&\qquad + \eta^2c_4\frac{\pmin}{p_j}\left(|\beta^*_{j}|_2^2 - |\Astar_i|_2^2 + I(i = i_{j})(|\Astar_i|_2^2 - |g^*_{j}|_2^2)\right) \\
&= -\eta^2c_4\frac{\pmin}{p_j}\left(|\Astar_i|_2^2 + |\beta^*_{j}|_2^2\right) + \eta^2 2c_4\frac{\pmin}{p_j}|g_{j}^*|_2^2\\
&\qquad + \eta^2c_4\frac{\pmin}{p_j}I(i = i_{j})\left(|\beta^*_{j}|_2^2 - |g^*_{j}|_2^2\right),
\end{split}
\end{equation*}
where we have plugged in $\Astar_{i_{j}} = \beta^*_{j}$, and in the case when $i = i_{j}$, the equality $\alpha_i = \beta^*_{j}$.
% \mkw{\textbf{I don't follow the second equality above.}}\mg{I added a few words of clarification...its seems to check out to me?}
Taking the expectation over $i$ yields 
\begin{equation*}
\begin{split}
\mathbb{E}_{i, j}[\phi_4(k + 1)] - \phi_4(k) &= -\eta^2c_4\pmin\left(\frac{m}{n}\sum_i\left[|\Astar_i|_2^2\right] + \sum_j\left[|\beta^*_{j}|_2^2\right]\right)\\
&\qquad+ \eta^2 2c_4\pmin\sum_j|g^*_{j}|_2^2\\
&\qquad+ \eta^2c_4\pmin\frac{m}{n}\left(\sum_j|\beta^*_{j}|_2^2 - \sum_j|g^*_{j}|_2^2\right) \\
&= - \eta^2 c_4\pmin\frac{m}{n}\sum_i|\Astar_i|_2^2\\
&\qquad -\eta^2c_4\pmin\left(1 - \frac{m}{n}\right)\sum_j|\beta^*_{j}|_2^2\\ 
&\qquad+ \eta^2c_4\pmin\left(2 - \frac{m}{n}\right)\sum_j|g^*_{j}|_2^2\\
&= -\frac{\eta^2c_4\pmin m}{4n}\left(2\sum_i|\Astar_i|_2^2 - \sum_j\left|\beta_{j}\right|_2^2\right)\\
&\qquad - \eta^2 c_4\frac{m\pmin}{n}\left(1 - \frac{2}{4}\right)\sum_i|\Astar_i|_2^2\\
&\qquad -\eta^2c_4\left(1 - \frac{3m}{4n}\right)\sum_j|\beta^*_{j}|_2^2\\ 
&\qquad+ \eta^2c_4\pmin\left(2 - \frac{m}{n}\right)\sum_j|g^*_{j}|_2^2 \\
&\leq -\frac{\pmin m}{4n}\phi_4(k)\\
&\qquad - \eta^2 c_4\frac{m\pmin}{2n}\sum_i|\Astar_i|_2^2\\
&\qquad -\eta^2c_4\pmin\left(1 - \frac{3m}{4n}\right)\sum_j|\beta^*_{j}|_2^2\\ 
&\qquad+ \eta^2c_4\pmin\left(2 - \frac{m}{n}\right)\sum_j|g^*_{j}|_2^2 \\
&\leq -\frac{\pmin m}{4n}\phi_4(k) - \eta^2c_4\frac{m \pmin}{2n}\sum_i|\Astar_i|_2^2\\
&\qquad -\eta^2\frac{c_4\pmin}{4}\sum_j|\beta^*_j|_2^2\\ 
&\qquad+ \eta^2 2\pmin c_4\sum_j|g^*_{j}|_2^2, \\
\end{split},
\end{equation*}
where the first inequality follows from the fact that $\frac{\pmin}{p_j} \leq p_j$ for all $j$, and the second inequality follows from bounding $0 \leq \frac{m}{n} \leq 1$.
\end{proof}

\begin{claim}\label{claim:phi_5}
\begin{equation*}
\mathbb{E}[\phi_5(k + 1)] - \phi_5(k) \leq c_5\eta^2\left[ - \pmin\sum_j|u_j|^2 + \frac{4m\pmin}{n}\sum_i\left(|\nabla_i - \nabla_i(x^*)|_2^2 + |\Astar_i|_2^2\right) + \frac{16m}{n\pmin}\sum_jp_j^2\left(|g_j^*|_2^2 + |\beta_j^*|_2^2\right)\right].
\end{equation*}
\end{claim}
\begin{proof}
If function $i$ is chosen in iteration $k$, then $u_{j}$ becomes $$u_j\left(1 - \frac{\pmin}{p_j}\right) + \frac{\pmin}{p_j}(\nabla_i - \alpha_i) - \left(\frac{m}{n}+ \frac{\pmin}{p_j}\left(I(i = i_j) - \frac{m}{n}\right)\right)h_{j}.$$

Hence in this case, by applying Jensen's inequality, we have

\begin{equation*}
\begin{split}
\phi_5(k + 1) - \phi_5(k) &= c_5\eta^2\left|u_j\left(1 - \frac{\pmin}{p_j}\right) + \frac{\pmin}{p_j}(\nabla_i - \alpha_i) - \left(\frac{m}{n}+ \frac{\pmin}{p_j}\left(I(i = i_j) - \frac{m}{n}\right)\right)h_{j}\right|^2 - c_5\eta^2|u_j|^2 \\
 &\leq c_5\eta^2\left[\frac{1}{1 - \frac{\pmin}{p_j}}\left|u_j\left(1 - \frac{\pmin}{p_j}\right)\right|^2 + \frac{1}{\frac{\pmin}{p_j}}\left|\frac{\pmin}{p_j}(\nabla_i - \alpha_i) - \left(\frac{m}{n}+ \frac{\pmin}{p_j}\left(I(i = i_j) - \frac{m}{n}\right)\right)h_{j}\right|^2 - |u_j|^2\right] \\
 &= c_5\eta^2\left[ - \frac{\pmin}{p_j}|u_j|^2 + \frac{p_j}{\pmin}\left|\frac{\pmin}{p_j}(\nabla_i - \alpha_i) - \left(\frac{m}{n}+ \frac{\pmin}{p_j}\left(I(i = i_j) - \frac{m}{n}\right)\right)h_{j}\right|^2\right].
\end{split}
\end{equation*}
Now we use Jensen's inequality and the fact that $\frac{\pmin}{p_j} \leq 1$ to break up the second squared term:
\begin{equation}
\begin{split}
&\left|\frac{\pmin}{p_j}(\nabla_i - \alpha_i) - \left(\frac{m}{n}+ \frac{\pmin}{p_j}\left(I(i = i_j) - \frac{m}{n}\right)\right)h_{j}\right|^2 \\
& \qquad \leq 4\frac{\pmin^2}{p_j^2}\left(|\nabla_i - \nabla_i(x^*)|_2^2 + |\Astar_i|_2^2\right) + 2\left(\frac{m}{n} + \frac{\pmin}{p_j}\left(I(i = i_j) - \frac{m}{n}\right)\right)^2|h_j|^2 \\
& \qquad \leq 4\frac{\pmin^2}{p_j^2}\left(|\nabla_i - \nabla_i(x^*)|_2^2 + |\Astar_i|_2^2\right) + 2\left(\frac{m}{n} + \frac{\pmin}{p_j}I(i = i_j)\right)^2|h_j|^2 \\
& \qquad \leq 4\frac{\pmin^2}{p_j^2}\left(|\nabla_i - \nabla_i(x^*)|_2^2 + |\Astar_i|_2^2\right) + 4\left(\frac{m}{n} + \frac{\pmin}{p_j}I(i = i_j)\right)^2|g_j^*|_2^2 + 4\left(\frac{m}{n} + \frac{\pmin}{p_j}I(i = i_j)\right)^2|\beta_j^*|_2^2.
\end{split}
\end{equation}

Observe that for a fixed $j$,

\begin{equation*}
\sum_{i: j(i) = j}\left(\frac{m}{n} + \frac{\pmin}{p_j}I(i = i_j)\right)^2 \leq \sum_{i: j(i) = j}\left(\frac{m}{n} + I(i = i_j)\right)^2 = \left(\frac{m}{n} + 1\right)^2 + \left(\frac{n}{m} - 1\right)\left(\frac{m}{n}\right)^2 \leq 4. 
\end{equation*}

Finally, taking the expectation over $i$, we have 
\begin{equation}
\begin{split}
\mathbb{E}[\phi_5(k+1)] & - \phi_5(k) \\
&\leq c_5\eta^2\left[ - \pmin\sum_j|u_j|^2 + \frac{4m\pmin}{n}\sum_i\left(|\nabla_i - \nabla_i(x^*)|_2^2 + |\Astar_i|_2^2\right) + \sum_j\frac{16mp_j^2}{n\pmin}\left(|g_j^*|_2^2 + |\beta_j^*|_2^2\right)\right] \\
% &\leq - \pmin\phi_5(k) + c_5\eta^2\left[\frac{4m\pmin}{n}\sum_i\left(|\nabla_i - \nabla_i(x^*)|_2^2 + |\Astar_i|_2^2\right) + \frac{16m \pmin}{n}\sum_j\left(|g_j^*|_2^2 + |\beta_j^*|_2^2\right)\right].
\end{split}
\end{equation}

\end{proof}

Combining Claim~\ref{claim:phi_3}, Claim~\ref{claim:phi_4}, and Claim~\ref{claim:phi_5} yields the following lemma.

\begin{customlem}{\ref{lem:phi_34}}(Formal)
\begin{equation*}
\begin{split}
\mathbb{E}[\phi_3(k + 1) + \phi_4(k + 1) + \phi_5(k + 1)] &- \phi_3(k) - \phi_4(k) - \phi_5(k)\\
&\leq -\frac{\pmin m}{4n}\left(\phi_3(k) + \phi_4(k) + \phi_5(k)\right) + \eta^2\frac{m\pmin}{n}\left(c_3 + 4c_5\right)\sum_{i}{\gradi}\\ 
&\qquad - \eta^2\frac{m}{n}\sum_i{\left(\pmin(4m\Lmain\eta + 16)\right)|\Astar_i|_2^2} - \eta^2\sum_j{\left(4m\Lmain\eta + 4\right)|u_j|_2^2}\\
&\qquad - \eta^2\sum_j{\left(\frac{16m \pmin}{n}\right)\left(|g_j^*|_2^2 + |\beta_j^*|_2^2\right)}.
\end{split}
\end{equation*}
\end{customlem}
\begin{proof}
\begin{equation*}
\begin{split}
&\mathbb{E}[\phi_3(k + 1) + \phi_4(k + 1) + \phi_5(k+1)] -\phi_3(k + 1) - \phi_4(k + 1) - \phi_5(k+1)   \\
&\leq  - \eta^2c_3\pmin\sum_j|g^*_j|_2^2 + \eta^2c_3\frac{m\pmin}{n}\sum_i\gradi \\
&\qquad -\frac{\pmin m}{4n}\phi_4(k) - \eta^2\frac{m \pmin c_4}{2n}\sum_i|\Astar_i|_2^2
 -\eta^2\frac{\pmin c_4}{4}\sum_j|\beta^*_j|_2^2 
+ \eta^2 2\pmin c_4\sum_j|g^*_{j}|_2^2\\
&\qquad + c_5\eta^2\left[ - \pmin\sum_j|u_j|^2 + \frac{4m\pmin c_5}{n}\sum_i\left(|\nabla_i - \nabla_i(x^*)|_2^2 + |\Astar_i|_2^2\right) + \sum_j\frac{16mp_j^2}{n\pmin}\left(|g_j^*|_2^2 + |\beta_j^*|_2^2\right)\right] \\
&= -\frac{\pmin m}{4n}\phi_4(k) \\ 
&\qquad + \eta^2\frac{m}{n}\sum_i{\left(-\frac{\pmin c_4}{2} + 4\pmin c_5\right)|\Astar_i|_2^2}\\
&\qquad + \eta^2\sum_j{\left(-\pmin c_3 + 2\pmin c_4 + \frac{16m p_j^2c_5}{n\pmin}\right)|g_j^*|_2^2}\\
&\qquad + \eta^2\sum_j{\left(-\frac{\pmin c_4}{4} + \frac{16m p_j^2c_5}{n\pmin}\right)|\beta_j^*|_2^2}\\
&\qquad + \eta^2\sum_j{\left(- c_5\pmin\right)|u_j|_2^2}\\
&\qquad + \eta^2\left(\frac{m\pmin}{n}c_3 + \frac{4m \pmin}{n}c_5\right)\sum_{i}{\gradi}\\ 
&= -\frac{\pmin m}{4n}\left(\phi_3(k) + \phi_4(k) + \phi_5(k)\right)\\ 
&\qquad + \eta^2\frac{m}{n}\sum_i{\left(-\frac{\pmin c_4}{2} + 4\pmin c_5\right)|\Astar_i|_2^2}\\
&\qquad + \eta^2\sum_j{\left(-\pmin c_3 + 2\pmin c_4 + \frac{16m p_j^2c_5}{n\pmin} + \frac{\pmin^2 m c_3}{4np_j}\right)|g_j^*|_2^2}\\
&\qquad + \eta^2\sum_j{\left(-\frac{\pmin c_4}{4} + \frac{16m p_j^2c_5}{n\pmin}\right)|\beta_j^*|_2^2}\\
&\qquad + \eta^2\sum_j{\left(- c_5\pmin + \frac{\pmin m c_5}{4n}\right)|u_j|_2^2}\\
&\qquad + \eta^2\left(\frac{m\pmin}{n}c_3 + \frac{4m \pmin}{n}c_5\right)\sum_{i}{\gradi}\\ 
&\leq -\frac{\pmin m}{4n}\left(\phi_3(k) + \phi_4(k) + \phi_5(k)\right)\\ 
&\qquad - \eta^2\frac{m}{n}\sum_i{\left(\pmin(4m\Lmain\eta + 16)\right)|\Astar_i|_2^2}\\
&\qquad - \eta^2\sum_j{\left(\frac{16m \pmin}{n}\right)|g_j^*|_2^2}\\
&\qquad - \eta^2\sum_j{\left(\frac{16m \pmin}{n}\right)|\beta_j^*|_2^2}\\
&\qquad - \eta^2\sum_j{m \pmin\left(4m\Lmain\eta + 4\right)|u_j|_2^2}\\
&\qquad + \eta^2\frac{m\pmin}{n}\left(c_3 + 4c_5\right)\sum_{i}{\gradi}.
\end{split}
\end{equation*}
Here the first inequality follows from Claims~\ref{claim:phi_3}, \ref{claim:phi_4}, \ref{claim:phi_5}, and the last inequality follows from plugging in our choice of constants: $c_5 = \frac{4}{3}\left(4m\Lmain\eta + 4\right)$, $c_4 = \left(22 + \frac{76m}{n}\left(\frac{\pmax}{\pmin}\right)^2\right)c_5$, and $c_3 = \left(64 + 168\frac{m}{n}\left(\frac{\pmax}{\pmin}\right)^2\right)c_5$.
\end{proof}

We use the following lemma to bound the $\sum_{i}\gradi $ terms, which appear in Lemmas~\ref{lem:phi_12} and \ref{lem:phi_34}.
\begin{lemma}\label{grad_diff_norm}
\begin{equation*}
\sum_i\gradi\leq  \Lsup y^T\nabla\mathbbm{1}.
\end{equation*}
\end{lemma}
\begin{proof}
By the convexity of each $f_i$ and their $\Lsup$-smoothness,
\begin{equation*}
\sum_i\gradi \leq \sum_{i}\Lsup y^T(\nabla_i - \nabla_i(x^*))= \Lsup y^T\nabla\mathbbm{1}.
\end{equation*}
Above, we used the fact that $\sum_i \nabla_i(x^*) = 0$.
\end{proof}

We now combine Lemma~\ref{lem:phi_12} and Lemma~\ref{lem:phi_34} to find the total expected difference in potential.
\begin{equation*}
\begin{split}
\mathbb{E}&[\phi(k + 1)]  - \phi(k)
\\ &= \left(\mathbb{E}[(\phi_1(k + 1) + \phi_2(k + 1))] - (\phi_1(k) + \phi_2(k))\right)+ \left(\mathbb{E}[(\phi_3(k + 1) + \phi_4(k + 1) + \phi_5(k + 1))] - (\phi_3(k) + \phi_4(k) + \phi_5(k))\right) \\
&\leq-2\pmin\begin{pmatrix}y \\ \eta (U\mathbbm{1} + m\Abar)\end{pmatrix}^T\begin{pmatrix}0 & 0 \\ 0 & 1\end{pmatrix}\begin{pmatrix}y\\ \eta(U\mathbbm{1} + m\Abar)\end{pmatrix} - \left(\frac{2\eta m\pmin}{n}\right)\left(y^T\nabla\mathbbm{1}\right)\\
&\qquad + \left(\frac{8m\pmin\eta^2}{n}\right)\sum_i{\gradi} \\
&\qquad - \frac{m\pmin}{4n}(\phi_3(k) + \phi_4(k) + \phi_5(k)) + \eta^2\frac{m\pmin}{n}(c_3 + 4c_5)\sum_i{\gradi}\\
&\leq -2\pmin\begin{pmatrix}y \\ \eta (U\mathbbm{1} + m\Abar)\end{pmatrix}^T\begin{pmatrix}0 & 0 \\ 0 & 1\end{pmatrix}\begin{pmatrix}y\\ \eta(U\mathbbm{1} + m\Abar)\end{pmatrix} - \frac{m\pmin}{4n}(\phi_3(k) + \phi_4(k) + \phi_5(k)) - \frac{m\pmin Cy^T\nabla\mathbbm{1}}{n},
\end{split}
\end{equation*}
where $$C = 2\eta - \eta^2 \Lsup\left(8 + c_3 + 4 c_5\right),$$ and in the last inequality, we have used Lemma~\ref{grad_diff_norm}.

Now because $f$ is convex, $f(x) - f(x^*) \leq \frac{y^T\nabla\mathbbm{1}}{n}$, and so rearranging terms and plugging in the value of $c_1$, we have

\begin{equation}\label{main_difference}
\begin{split}
\mathbb{E}[\phi(k + 1)] - \phi(k) &\leq -2\pmin\begin{pmatrix}y \\ \eta (U\mathbbm{1} + m\Abar)\end{pmatrix}^T\begin{pmatrix}0 & 0 \\ 0 & 1\end{pmatrix}\begin{pmatrix}y\\ \eta(U\mathbbm{1} + m\Abar)\end{pmatrix} - \frac{m\pmin}{4n}(\phi_3(k) + \phi_4(k) + \phi_5(k)) \\
&\qquad - \frac{c_1 m \pmin}{4n}(f(x) - f(x^*)) - m\pmin\left(C - \frac{c_1}{4n}\right)\frac{y^T\nabla\mathbbm{1}}{n} \\
&= -2\pmin\begin{pmatrix}y \\ \eta (U\mathbbm{1} + m\Abar)\end{pmatrix}^T\begin{pmatrix}0 & 0 \\ 0 & 1\end{pmatrix}\begin{pmatrix}y\\ \eta(U\mathbbm{1} + m\Abar)\end{pmatrix} - \frac{m\pmin}{4n}(\phi_1(k) + \phi_3(k) + \phi_4(k) + \phi_5(k)) \\
&\qquad - m\pmin\left(C - \frac{c_1}{4n}\right)\frac{y^T\nabla\mathbbm{1}}{n}.
\end{split}
\end{equation}

The next claim shows that for small enough $\eta$, the final term in this equation is large.

\begin{claim}\label{eq:det}
For $\eta \leq \etabd$, 
$$C - \frac{c_1}{4n} \geq \eta,$$

where $r = \frac{8\left(76 + 168\left(\frac{\pmax}{\pmin}\right)^2\frac{m}{n}\right)}{3}$ as in Theorem~\ref{thm:main}.

\end{claim}
\begin{proof}
First observe that

$$r \geq \frac{8 + \left(68 + 168\left(\frac{\pmax}{\pmin}\right)^2\frac{m}{n}\right)\frac{4}{3}\left(4m\Lmain\eta + 4\right)}{2(m\Lmain\eta + 1)} = \frac{\Lsup(8 + c_3 + 4c_5)}{2\Lsup\left(m\Lmain\eta + 1\right)}.$$

Thus 
\begin{equation*}
\begin{split}
C - \frac{c_1}{4n} &\geq 2\eta - 2\eta^2r(m\Lmain\eta + 1) \\
&= 2\eta(1 - \eta r\Lsup - \eta^2 r m \Lmain\Lsup)\\
& \geq \eta.
\end{split}
\end{equation*}
for $\eta \leq \frac{1}{2r\Lsup + 2\sqrt{rm\Lmain\Lsup}}$. Here in the last line we have used the fact that $(1 - \eta r\Lsup - \eta^2 r m \Lmain\Lsup)$ is increasing in $\eta$, and for any $a, b > 0$, $1 - \frac{a}{2(a + b)} - \frac{b^2}{(2(a + b))^2} \geq \frac{1}{2}$ (we plugged in $a = r\Lsup$ and $b^2 = rm\Lmain \Lsup$).

\end{proof}

Using the strong convexity of $f$, we have $y^T\nabla\mathbbm{1} \geq n \mu |y|_2^2.$ Hence plugging Claim~\ref{eq:det} into \eqref{main_difference} yields the following lemma.

\begin{customlem}{\ref{lem:phi_all}}
For $\eta < \etabd$,
\begin{equation*}
\begin{split}
\mathbb{E}[\phi(k + 1)] - \phi(k) \leq -2\pmin\begin{pmatrix}y \\ \eta (U\mathbbm{1} + m\Abar)\end{pmatrix}^T\begin{pmatrix}\frac{\mu m \eta}{2} & 0 \\ 0 & 1\end{pmatrix}\begin{pmatrix}y\\ \eta(U\mathbbm{1} + m\Abar)\end{pmatrix} - \frac{m\pmin}{4n}(\phi_1(k) + \phi_3(k) + \phi_4(k) + \phi_5(k))
\end{split}
\end{equation*}
\end{customlem}

Recall that our goal is to find some $\gamma \leq \frac{m\pmin}{4n}$ such that 
$$\mathbb{E}[\phi(k+1)] \leq \left(1 - \gamma\right)\phi(k).$$

We will do this by finding some $\gamma$ that satisfies for all $y, U$ and $\Abar$:
\begin{equation}\label{desire}
\frac{\begin{pmatrix}y \\ \eta(U\mathbbm{1} + m\Abar)\end{pmatrix}^T\begin{pmatrix}\frac{\mu m \eta}{2} & 0 \\ 0 & 1\end{pmatrix}\begin{pmatrix}y\\\eta(U\mathbbm{1} + m\Abar)\end{pmatrix}}{\begin{pmatrix}y \\ \eta(U\mathbbm{1} + m\Abar)\end{pmatrix}^T\begin{pmatrix}1 & -1 \\ - 1& 2\end{pmatrix}\begin{pmatrix}y\\ \eta(U\mathbbm{1} + m\Abar)\end{pmatrix}} \geq \frac{\gamma}{2\pmin},
\end{equation}

or equivalently
\begin{equation*}
    Q \succeq \frac{\gamma}{\pmin} I,
\end{equation*}
where 
\begin{equation}\label{eq:mat_prod}
    Q := \begin{pmatrix}1 & -1 \\ - 1& 2\end{pmatrix}^{-1/2}\begin{pmatrix}\mu m \eta & 0 \\ 0 & 2\end{pmatrix}\begin{pmatrix}1 & -1 \\ - 1& 2\end{pmatrix}^{-1/2}.
\end{equation}
Indeed, establishing \eqref{desire} will imply that 
$\mathbb{E}[\phi(k + 1)] - \phi(k) \leq -\gamma\phi_2(k) - \frac{m\pmin}{4}(\phi_1(k) + \phi_3(k) + \phi_4(k) + \phi_5(k)),$ so if $\gamma \leq \frac{m\pmin}{4n}$, then $\mathbb{E}[\phi(k)] \leq \left(1 - \gamma\right)\phi(k)$.

We bound the smallest eigenvalue of $Q$ by evaluating the trace and determinant of the product of $2 \times 2$ matrices that underlie the block matrices above in the matrix product forming $Q$.

\begin{lemma}\label{lambda_min_fact}
For any symmetric $2 \times 2$ matrix A, $\lambda_{min}(A) \geq \frac{\Det(A)}{\Tr(A)}$.
\end{lemma}
\begin{proof}
Let $d := \Det(A)$, and  $t := \Tr(A)$. By the characteristic equation, putting  \begin{equation*}
\lambda_{min}(A) = \frac{t - \sqrt{t^2 - 4d}}{2} = \frac{t}{2}\left(1 - \sqrt{1 - \frac{4d}{t^2}}\right) \geq \frac{t}{2}\left(1 - \left(1 - \frac{2d}{t^2}\right)\right) = \frac{d}{t},
\end{equation*}
where the inequality follows from the fact that $\sqrt{1 + x} \leq 1 + \frac{x}{2}$ for $x \geq -1$.
\end{proof}

\begin{claim}\label{minQ}
For $\eta \leq \etabd$,
\begin{equation*}
\lambda_{min}(Q) \geq \min\left(1, \mu m \eta \right).
% \lambda_{min}(Q) \geq \min\left(\frac{1}{4}, \frac{3n\mu \eta}{8}\right).
\end{equation*}
% where 
% \begin{equation}
% D := \Det(Q) = \frac{2n^2}{m}\eta\mu\left(2 - \frac{m}{n} - 8\eta\Lsup - \eta c_3\Lsup\right),
% \end{equation}
% assuming $D > 0$.
\end{claim}
\begin{proof}

We compute the determinant and trace of $Q$. Note that $\det\left(\begin{pmatrix}1 & - 1 \\ -1 & 2\end{pmatrix}\right) = 1$.

\begin{equation*}
\begin{split}
D := \Det(Q) &= 2\mu m \eta\\ 
% &= \frac{\frac{2n^2}{m}\eta\mu(c_3 - c_1)(2c_1 - \frac{m c_3}{2n} - 4\eta c_3\Lsup - \eta c_3\Lsup)}{c_1(c_3 - c_1)}\\
% &= \frac{2n^2}{m}\eta\mu\left(2 - \frac{m}{n} - 8\eta\Lsup - \eta c_3\Lsup\right).
\end{split}
\end{equation*}
Using the circular law of trace, and computing the inverse $\begin{pmatrix}1 & - 1 \\ -1 & 2\end{pmatrix}^{-1} = \begin{pmatrix}2 & 1 \\ 1 & 1\end{pmatrix}$, we have
\begin{equation*}
\begin{split}
\Tr(Q) &= 2\mu n \eta + 2 = D + 2.
\end{split}
\end{equation*}

Now by Lemma~\ref{lambda_min_fact}
\begin{equation*}
\lambda_{min}(Q) \geq \frac{D}{D + 2} \geq \min\left(1, \frac{D}{2} \right) = \min\left(1, \mu m \eta \right), 
% \lambda_{min}(Q) \geq \frac{D}{\frac{1}{n \pmin}D + 2 m \pmin} \geq \min\left(\frac{1}{4}, \frac{3 Dn}{8n} \right) = \min\left(\frac{1}{4}, \frac{3n\mu\eta}{8} \right), 
\end{equation*}
where the second inequality holds for any $D > 0$.
\end{proof}

Recalling our bound that $\gamma \leq \frac{m\pmin}{4n}$, this claim shows that we can choose $\gamma = m\pmin\min\left(\frac{1}{4n}, \mu\eta\right)$ as desired.
\end{proof}

%  \mkw{I'm conflicted about the notation $\mathbb{E}_{i, j}[g_j]$ (as opposed to $\mathbb{E}_{i, j}[ g_{j}]$). On the one hand, the first is shorter; on the other hand, the second makes it clear where the $i$ in the expectation comes in.}\mg{As the writer, I obviously understand this expectation well and I find it hard to look at complicated subscripts, but its hard to tell whats best for the reader.}\mkw{It's probably fine to leave it as $\mathbb{E}_{i, j} g_j$.  The reader will get used to it.}
 
%  \mg{Figure out where to put this remark: Notice that while the variable $i_j$ is used to compute the updates, its value is not used in the potential function, nor will the expected potential at future iterations depend on this value.}

\section{Minibatch Rates}
We prove minibatch rates for SAGA in this appendix for both the shared and distributed data setting. In the shared data setting, minibatch SAGA is an instantiation Algorithm~\ref{alg:minibatch} with $S_j = [n]$ for all $j \in [m]$. In the distributed data setting, $S_j$ is the set of functions held at machine $j$.

\begin{algorithm}[t]
  \caption{Minibatch SAGA}
  \label{alg:minibatch}
\begin{algorithmic}
\Procedure{MinibatchSAGA}{ $x, \eta, \{f_i\}, \{S_j\}, t$}
\State $\alpha_i =0$ for $i \in [n]$\Comment{Initialize last gradients to 0 at each machine}
\State $\overline{\alpha} =0$ \Comment{Initialize last gradient averages at PS}
\For{$k=0$ {\bfseries to} $t$}
\EndFor
\For{$j = 1$ {\bfseries to} $m$}
\State $i_j \sim \text{Uniform}(S_j)$ \Comment{Randomly choose a function}
\State $g \leftarrow \sum_j{\nabla f_{i_j}(x)}$ \Comment{Compute the minibatch gradient}
\State $\beta \leftarrow \sum_j{\alpha_{i_j}(x)}$ \Comment{Compute the variance reduction term}
\EndFor
\For{$j = 1$ {\bfseries to} $m$}
\State $\alpha_{i_j} \leftarrow \nabla f_{i_j}(x)$  \Comment{Update the relevant $\alpha_i$ variables}
\State $x \leftarrow x - \eta(g - \beta + m\Abar)$ \Comment{Take a gradient step}
\State $\Abar \leftarrow \mathbb{E}_i[\alpha_i]$ \Comment{Update $\Abar$}
\EndFor \\
\Return{$x$}
\EndProcedure
\end{algorithmic}
\end{algorithm}

% \begin{algorithm}
%   \caption{Minibatch SAGA}
%   \label{alg:minibatch}
% \SetAlgoLined
% \KwIn{$x, \eta, \{f_i\}, \{S_j\}, t$}\;
% $\alpha_i = 0$ for $i \in [n]$\;\\
% $\Abar =0$\\
% \For{$k=0$ \KwTo $t$}{
% \For{$j = 1$ \KwTo $m$}{
% $i_j \sim \text{Uniform}(S_j)$ \tcc*{Randomly choose a function}
% }\;
% $h \leftarrow \sum_j{\nabla f_{i_j}(x)}$\tcc*{Compute the minibatch gradient}\;
% $\beta \leftarrow \sum_j{\alpha_{i_j}(x)}$\tcc*{Compute the variance reduction term}
% \For{$j = 1$ \KwTo $m$}{$\alpha_{i_j} \leftarrow \nabla f_{i_j}(x)$  \tcc*{Update the relevant $\alpha_i$ variables}
% }\;
% $x \leftarrow x - \eta(h - \beta + m\Abar)$\; \tcc*{Take a gradient step}
% $\Abar \leftarrow \mathbb{E}_{i, j}[\alpha_i]$\; \tcc*{Update $\Abar$}
% }
% \KwRet{$x$}
% \end{algorithm}

\begin{proposition}\label{prop:minibatch}
Let $f(x) = \frac{1}{n}\sum_i f_i(x)$ be $\mu$-strongly convex and $\Lmain$-smooth. Suppose each $f_i$ is convex and $\Lsup$-smooth. Let $\sigma^2 := \mathbb{E}_{i, j}{|\nabla f_i(x^*)|_2^2}$.

Consider the minibatch SAGA algorithm (Algorithm~\ref{alg:minibatch}) with a minibatch size of $m$ for either the shared data or distributed data setting with $m$ machines. Then with a step size of $\eta = \frac{1}{2m\Lmain + 3\Lsup}$, after $$\left(\frac{3n}{m} + \frac{12\Lsup}{m \mu} + \frac{4\Lmain}{\mu}\right)\log\left(\frac{|x^0 - x^*|_2^2 + \frac{4n\sigma^2}{\left(2m\Lmain + 3\Lsup\right)^2}}{\epsilon}\right)$$ iterations, ie. $\tilde{O}\left(n + \frac{\Lsup}{\mu} + \frac{m\Lmain}{\mu}\log(1/\epsilon)\right)$ total gradient computations, we have $$|x^k - x^*| \leq \epsilon.$$
\end{proposition}

\begin{proof}
For convenience, we define $y^k := x^k - x^*$, and $\nabla_i := \nabla f_i (x)$. When the iteration $k$ is clear from context, we omit the superscript $k$.

Consider the following potential function 
\begin{equation*}
\phi(x, \alpha) := \phi_1(x, \alpha) + \phi_2(x, \alpha),
\end{equation*}
where
\begin{equation*}
\begin{split}
\phi_1(x, \alpha) &:= |x - x^*|_2^2; \\
\phi_2(x, \alpha) &:= 4n\eta^2\mathbb{E}_{i, j}\left[\left|\alpha_i - \nabla_i(x^*)\right|_2^2\right].
\end{split}
\end{equation*}
For convenience we denote $\phi(k) := \phi(x^k, \alpha^k)$.

% \begin{lemma}\label{sg_norm}
% \begin{equation*}
% \mathbb{E}_{i, j}[|\nabla_i - \nabla_i(x^*)|_2^2]\leq  \Lsup y^T\nabla f(x).
% \end{equation*}
% \end{lemma}
% \begin{proof}
% By the convexity of each $f_i$ and their $\Lsup$-smoothness,
% \begin{equation*}
% \mathbb{E}_{i, j}[|\nabla_i - \nabla_i(x^*)|_2^2] \leq \mathbb{E}_{i, j}[\Lsup y^T(\nabla_i - \nabla_i(x^*))] = \Lsup y^T\nabla f(x).
% \end{equation*}
% \end{proof}

Let $B := \{i_j\}_{j \in [m]}$ be the minibatch chosen at iteration $k$. Let $\mathbbm{1}_B \in \mathbb{N}^n$ be the indicator vector for the multi-set $B$. Let $S(B)$ be the set containing all elements of $B$ and $\mathbbm{1}_{S(B)}$ the corresponding indicator vector. Let $\mathbbm{1}_{S_j} \in \{0, 1\}^n$ be the indicator vector of $S_j$, the data points at machine $j.$

\begin{lemma}
$$\mathbb{E}_{B}[\phi_1(k + 1)] - \phi_1(k) \leq - 2m\eta y^T\nabla f(x) + \eta^2m^2|\nabla f(x)|_2^2 + 2\eta^2m\Lsup y^T\nabla f(x) + 2\eta^2m\mathbb{E}_{i, j}\left[|\alpha_i - \nabla_i(x^*)|_2^2\right].$$
\end{lemma}
\begin{proof}
Let $\alpha$ denote the matrix whose $i$th column is $\alpha_i$. We abuse notation by using $\Abar$ to also mean the matrix $\mathbbm{1}_n\Abar^T$.
We have
\begin{equation*}
\phi_1(k + 1) - \phi_1(k) = - 2m \eta y^T(\nabla - \alpha + \Abar)\mathbbm{1}_{B} + \eta^2\mathbbm{1}_{B}^T(\nabla - \alpha + \Abar)^T(\nabla - \alpha + \Abar)\mathbbm{1}_{B},
\end{equation*}
so 
\begin{equation}\label{eq:expansion}
\begin{split}
\mathbb{E}_B[\phi_1(k + 1)] - \phi_1(k) &= - \frac{2m}{n} \eta y^T\nabla f(x) + \eta^2\mathbb{E}_B\left[\mathbbm{1}_{B}^T(\nabla - \alpha + \Abar)^T(\nabla - \alpha + \Abar)\mathbbm{1}_{B}\right] \\
&= - \frac{2m}{n} \eta y^T\nabla f(x) + \eta^2\Tr\left(\mathbb{E}_B\left[\mathbbm{1}_B\mathbbm{1}_{B}^T\right](\nabla - \alpha + \Abar)^T(\nabla - \alpha + \Abar)\right),
\end{split}
\end{equation}
where the second line follows from the circular law of trace. 
\end{proof}

Consider first the shared data case. Here we have 
$$\mathbb{E}_B\left[\mathbbm{1}_B\mathbbm{1}_{B}^T\right] = \frac{m(m - 1)}{n(n - 1)}\mathbbm{1}\mathbbm{1}^T + \left(\frac{m}{n} - \frac{m(m - 1)}{n(n - 1)}\right)I_n .$$

In the distributed data case, we have 
$$\mathbb{E}_B\left[\mathbbm{1}_B\mathbbm{1}_{B}^T\right] = \frac{m^2}{n^2}\mathbbm{1}\mathbbm{1}^T - \frac{m^2}{n^2}\sum_j{\mathbbm{1}_{S_j}\mathbbm{1}_{S_j}^T} + \frac{m}{n}I_n \preceq \frac{m^2}{n^2}\mathbbm{1}\mathbbm{1}^T + \frac{m}{n}I_n.$$

Hence by linearity of the trace operator, in both cases,
\begin{equation*}
\begin{split}
&\Tr\left(\mathbb{E}_B\left[\mathbbm{1}_B\mathbbm{1}_{B}^T\right](\nabla - \alpha + \Abar)^T(\nabla - \alpha + \Abar)\right)\\
&\leq  \frac{m^2}{n^2}\Tr\left(\mathbbm{1}\mathbbm{1}^T(\nabla - \alpha + \Abar)^T(\nabla - \alpha + \Abar)\right) + \frac{m}{n}\Tr\left((\nabla - \alpha + \Abar)^T(\nabla - \alpha + \Abar)\right) \\
&= \frac{m^2}{n^2}\mathbbm{1}^T(\nabla - \alpha + \Abar)^T(\nabla - \alpha + \Abar)\mathbbm{1} + \frac{m}{n}\sum_i{|\nabla_i - \alpha_i + \Abar|_2^2} \\
&\leq m^2|\nabla f(x)|_2^2 + 2m\mathbb{E}_{i, j}\left[\nabla_i - \nabla_i(x^*)\right] + 2m\mathbb{E}_{i, j}\left[\alpha_i - \nabla_i(x^*) + \Abar\right] \\
&\leq m^2|\nabla f(x)|_2^2 + 2m\mathbb{E}_{i, j}\left[\nabla_i - \nabla_i(x^*)\right] + 2m\mathbb{E}_{i, j}\left[\alpha_i - \nabla_i(x^*)\right],
\end{split}
\end{equation*}
where the third line follows from the circular law of trace, the fourth line from Jensen's inequality, and the fifth line from the positivity of variance, that is $\mathbb{E}_{i, j}\left[\alpha_i - \nabla_i(x^*) + \Abar\right] = \mathbb{E}_{i, j}\left[\alpha_i - \nabla_i(x^*)\right] + |\Abar|_2^2$.

Plugging in Lemma~\ref{grad_diff_norm} and combining with \eqref{eq:expansion} yields the lemma.

\begin{lemma}
$$\mathbb{E}_{B}[\phi_2(k + 1)] - \phi_2(k) \leq -\left(1 - \exp\left(-\frac{m}{n}\right)\right)\phi_2(k) + 4m\eta^2\Lsup y^T\nabla f(x)$$
\end{lemma}
\begin{proof}
\begin{equation}\label{eq:squared_term}
\phi_3(k+1) - \phi_3(k) = 4n\eta^2\left(\sum_{i \in S(B)}{\left(|\nabla_i - \nabla_i(x^*)|_2^2 - |\alpha_i - \nabla_i(x^*)|_2^2\right)}\right)
\end{equation}

In the shared data setting, $\mathbb{E}_B[\mathbbm{1}_{S(B)}] = \left(1 - \left(1 - \frac{1}{n}\right)^m\right)\mathbbm{1}$. In the distributed data setting, $\mathbb{E}_B[\mathbbm{1}_{S(B)}] = \frac{m}{n}\mathbbm{1}$.

Now $$\left(1 - \exp\left(-\frac{m}{n}\right)\right) \leq \left(1 - \left(1 - \frac{1}{n}\right)^m\right) \leq \frac{m}{n}.$$ Plugging these bounds into \eqref{eq:squared_term} with Lemma~\ref{grad_diff_norm} directly yields the lemma.
\end{proof}

Combining these two lemmas, we get 
\begin{equation*}
\mathbbm{E}[\phi(k + 1)] - \phi(k) \leq (- 2\eta m + 6m \eta^2 \Lsup)y^T\nabla f(x) + \eta^2m^2|\nabla f(x)|_2^2 -\left(1 - \frac{m}{2n}- \exp\left(-\frac{m}{n}\right)\right)\phi_2(k).
\end{equation*}

Now for $\eta < \frac{1}{6\Lsup}$, we have 
\begin{equation*}
\begin{split}
(- 2m\eta + 6m \eta^2 \Lsup)y^T\nabla f(x) + \eta^2m^2|\nabla f(x)|_2^2 & \leq -\eta m y^T\nabla f(x) + \eta^2m^2|\nabla f(x)|_2^2.\\
\end{split}
\end{equation*}

Further, if $\eta < \frac{1}{2m\Lmain}$, by the $\Lmain$-smoothness of $f$, we have 
\begin{equation*}
-\eta m y^T\nabla f(x) + \eta^2m^2|\nabla f(x)|_2^2 \leq - \frac{\eta m}{2}y^T\nabla f(x).
\end{equation*}

By the $\mu$-strong convexity of $f$, we have 
\begin{equation*}
- \frac{\eta m}{2}y^T\nabla f(x) \leq - \frac{\mu \eta m}{2}|y|_2^2 = \frac{\mu \eta m}{2}\phi_1(k).
\end{equation*}

Further, $\left(1 - \frac{m}{2n}- \exp\left(-\frac{m}{n}\right)\right) \geq \frac{m}{3n}$. It follows that for $\eta \leq \frac{1}{2m\Lmain + 6\Lsup}$,
\begin{equation}\label{eq:potential_change}
\mathbbm{E}[\phi(k + 1)] - \phi(k) \leq \frac{\eta \mu m}{2}\phi_1(k) - \frac{m}{3n}\phi_2(k).
\end{equation}

Choosing $\eta =  \frac{1}{2m\Lmain + 6\Lsup}$, it follows that 
\begin{equation*}
\mathbb{E}_B[\phi(k + 1) | \phi(k) ] \leq \left(1 - \gamma\right)\phi(k),
\end{equation*}
where $\gamma = \min\left(\frac{\mu m}{4m\Lmain + 12\Lsup}, \frac{m}{3n}\right)$.

Hence after $k = \frac{1}{\gamma}\log\left(\frac{\phi(0)}{\epsilon}\right)$ iterations, we have 
\begin{equation*}
\mathbb{E}[\phi(k)] \leq \left(1 - \gamma\right)^k\phi(0) \leq \exp(-\gamma k)\phi(0) = \epsilon.
\end{equation*}

Now $\phi(0) = |x^0 - x^*|_2^2 + \frac{4n\sigma^2}{\left(2m^2\Lmain + 3m\Lsup\right)^2}$,
and $\phi(k) \geq |x^k - x^*|_2^2$,

so after $k = \left(\frac{3n}{m} + \frac{12\Lsup}{m \mu} + \frac{4\Lmain}{\mu}\right)\log\left(\frac{|x^0 - x^*|_2^2 + \frac{4n\sigma^2}{\left(2m\Lmain + 3\Lsup\right)^2}}{\epsilon}\right)$ iterations, we have 
\begin{equation*}
\mathbb{E}[|x^k - x^*|] \leq \epsilon.
\end{equation*}
\end{proof}

\end{document}